\documentclass{article}

\usepackage{amsmath}
\usepackage{amssymb}
\usepackage{amsthm}
\usepackage{style}
\usepackage[round]{natbib}
\usepackage[margin=1in]{geometry}
\usepackage[colorlinks=true,linkcolor=blue,citecolor=blue]{hyperref}
\newtheorem{theorem}{Theorem}
\newtheorem{lemma}{Lemma}
\newtheorem{definition}{Definition}

\begin{document}

\title{Online Recommendations for Agents with Discounted Adaptive Preferences}
\author{Arpit Agarwal\thanks{\texttt{agarpit@outlook.com}} \and William Brown\thanks{\texttt{w.brown@columbia.edu}}}
\date{\today}

\maketitle

\begin{abstract}
We consider a bandit recommendations problem in which an agent’s preferences (representing selection probabilities over recommended items) evolve as a function of past selections, according to an unknown \textit{preference model}. In each round, we show a menu of $k$ items (out of $n$ total) to the agent, who then chooses a single item, and we aim to minimize regret with respect to some \textit{target set} (a subset of the item simplex) for adversarial losses over the agent’s choices. Extending the setting from \cite{AgarwalB22}, where uniform-memory agents were considered, here we allow for non-uniform memory in which a discount factor is applied to the agent’s memory vector at each subsequent round. In the ``long-term memory'' regime (when the effective memory horizon scales with $T$ sublinearly), we show that efficient sublinear regret is obtainable with respect to the set of \textit{everywhere instantaneously realizable distributions} (the ``EIRD set'', as formulated in prior work) for any \textit{smooth} preference model. Further, for preferences which are bounded above and below by linear functions of memory weight (we call these ``scale-bounded'' preferences) we give an algorithm which obtains efficient sublinear regret with respect to nearly the \textit{entire} item simplex. We show an NP-hardness result for expanding to targets beyond EIRD in general. In the “short-term memory” regime (when the memory horizon is constant), we show that scale-bounded preferences again enable efficient sublinear regret for nearly the entire simplex even without smoothness if losses do not change too frequently, yet we show an information-theoretic barrier for competing against the EIRD set under arbitrary smooth preference models even when losses are constant.
\end{abstract}

\section{Introduction}

Recommendation systems are an integral part of 
online platforms for  
e-commerce, social networks, and content sharing.
It it well-documented that user preferences change over time in response to content recommendations \citep{flaxman,CurmeiHRH22,abdollahpouri2019managing,DBLP:journals/corr/abs-2007-13019}, leading to self-reinforcing patterns of content consumption which  can have a variety of unintended consequences for the user, platform, and society, including but not limited to loss of revenue and creation of ``filter-bubbles'' and ``echo-chambers'' that drive polarization.
Hence, it is important for recommendation systems to incorporate these adaptivity patterns into user modeling to capture the long-run effects of recommendations on agent preference dynamics.

The approach we take to this problem, building on the framework from \cite{AgarwalB22}, is to formulate online recommendation as a nonlinear reinforcement learning problem over the space of agent preferences. 
Often the objective of recommendation systems is to optimize some function of agent behavior, e.g.\ item purchases or ad revenue, but we cannot force the agent to choose any individual option we present them with. When preferences are adaptive, the menu of recommendations we present to a agent in each interaction determines not only our expected immediate rewards (given the agent's current choice distribution over the menu), but also affects the agent's downstream choice likelihoods over subsequent recommendations (via updates to their preferences), and hence our future landscape of feasible rewards. 
This causal effect occurs on a per-agent basis, breaking the distributional assumptions required by classical approaches such as collaborative filtering. This motivates the importance of establishing rigorous foundations for the dynamics of such problems at the level of individual agents. To evaluate long-run interactions, the framework of regret minimization is then natural to consider, which further allows accommodation of adversarial changes to per-item rewards over time.
While some recent work has adopted the regret minimization perspective for multi-agent recommendation problems \citep{GaitondeKT21,DeanM22,jagadeesan2022supplyside}, strong assumptions are required on the recommendation setting and model of preference evolution, in which only a single recommendation is given each round, which then results in a linear update to preferences. In contrast, the setting we consider allows for unknown and potentially highly nonlinear preference dynamics which can express the complex interactions often present between items (e.g.\ substitute and complement effects, relevant sequential orderings, or genre correlations), as well as to accommodate the practical constraint faced by many systems in which an agent must be shown a menu of multiple items to choose from.
This yields a problem which is non-trivial even in the case of a single agent.

While the setting of \cite{AgarwalB22} accommodates many of these considerations by allowing for multi-item recommendation menus and nonlinear preference dynamics, several major limitations remain present. There, it is assumed that the agent's relative preferences for each item, as functions of the empirical distribution of past item selections, are determined by functions from a known parametric class which satisfies a somewhat strong learnability condition (including e.g.\ constant-degree polynomials), and their algorithm involves learning each function  explicitly. Further, the target set considered for regret minimization is fairly restrictive, only containing item distributions which are always feasible regardless of the agent's current preferences (when an appropriate menu distribution is chosen). We show that both of these restrictions can be removed nearly entirely, 
while also enabling flexibility for the speed at which agent preferences may evolve,
with only minor adjustments which are well-justified from the perspective of preference modeling. 

Our first change is to allow for agents with {\it non-uniform} memory, where a discount factor $\gamma \leq 1$ is applied to the vector of past selections in each round. Initially we assume that $\gamma$ approaches $1$ as $T$ grows, yielding an effective memory horizon (i.e.\ the window of rounds whose influence dominates memory) which is $\Theta(T^c)$, for some $c \in (0, 1]$. For such cases, we show that the parametric learnability assumption can be replaced by assuming only that each function is Lipschitz, and we give an algorithm which obtains sublinear regret with respect to same target set considered by \cite{AgarwalB22}, recovering their $T^{3/4}$ rate in the uniform-memory case (when $\gamma = 1$). 
Next, we identify a property for the agent's preference functions which captures a regime of interpolation between linearly increasing and arbitrary functions, with a multiplicative factor bounding the distance from linearity at each point; we say that such functions are ``scale-bounded''. For agents with such preferences and {\it strictly} sublinear memory horizons (i.e.\ with $c \in (0,1)$) we give an algorithm which obtains sublinear regret with respect to nearly the {\it entire} item simplex. The distance from our target set here to the  simplex boundary is a function of the agent's ``exploration term'' $\lambda > 0$, which may be arbitrarily small in this case. These conditions can be interpreted as requiring that (i) agents are recency-biased, and (ii) preferences generally increase with familiarity, both of which are widely assumed in modeling of agent preferences (see \cite{CurmeiHRH22} for an overview). We give a number of other results regarding efficient menu selection and  constant-memory agents, as well as negative results for further relaxations.

\subsection{Setting}

As in \cite{AgarwalB22}, in each round $t\in [T]$, we (the recommender) must choose $k$ out of $n$ items to show in a menu $K_t$ to the agent, who will choose one of the $k$ items probabilistically.
The choice probability of 
each item depends
on the agent's {\it memory vector} $v_t \in \Delta(n)$, which encodes their history of prior item selections,
as well as their {\it preference model}, which maps memory vectors to preference {\it scores} via functions $f_i(v_t) \in [0,1]$ for each item; selection probabilities for each $i\in K_t$ are proportional to these scores.
We receive bandit adversarial rewards $r_t(i_t) \in [0,1]$ for the agent's item choice $i_t$ in each round. 

Observe that for $k=1$ this collapses to the classical adversarial multi-armed bandit problem, as the agent will choose the single shown item deterministically. Here we take $k > 1$ to be fixed as an input to the problem, which yields drastically more intricate selection dynamics than single-item recommendations due to the role of the preference model; 
we are playing a bandit problem in which we cannot pull arms (items) directly, and the set of {\it instantaneously realizable distributions} of items we may induce (when considering any distribution over menus we may choose in each round) is shifting as a function of the history according to unknown and nonlinear dynamics. As a result, in general it will not be possible to do as well as the best individual item in terms of reward. Even if scores were always {\it uniform}, an agent
would never choose an item more than once every $k$ rounds even if it was shown in every menu; hence, we must be careful in selecting benchmark {\it target sets} (as subsets of the item simplex) with which we can feasibly compete in terms of sublinear regret.

\subsection{Our Results}

We give a series of algorithmic and barrier results for the online recommendations problem of \cite{AgarwalB22} for {\it discounted}-memory agents, where the memory vector updates as $v_{t+1} \propto i_t  + \gamma v_t/(1 - \gamma)$, for some $\gamma \in [0,1]$. We formalize the recommender-agent interaction model and introduce key preliminaries in Section \ref{sec:prelims}.

In Section \ref{sec:biggamma}, we give an algorithm which obtains $o(T)$ regret to the ``$\EIRD$ set'' (as introduced by \cite{AgarwalB22})
for agents with an effective memory horizon of $\Omega(T^c)$ for any $c \in (0,1]$. This target set consists of the {\it everywhere instantaneously realizable distributions} of items, i.e.\ item distributions $p_t \in \Delta(n)$ such that for any memory vector $v_t \in\Delta(n)$ and resulting preference set $\{f_i(v_t)\}$ there is some menu distribution $z_t$ which yields $i_t \sim p_t$ in expectation. Here, we also introduce a new characterization of the set of instantaneously realizable distributions for any $v_t$,  enabling a computationally efficient menu selection step which removes the exponential dependence on $k$ from the linear programming subroutine in prior work.

In Section \ref{sec:biggamma-pi}, we formulate our notion of {\it scale-bounded preference models} and consider a target set which we call the {\it $\phi$-smoothed simplex}, denoted $\Delta^{\phi}(n)$, corresponding to any point in the simplex $\Delta(n)$ mixed with $\phi$ uniform noise. We give an algorithm which obtains $o(T)$ regret with respect to $\Delta^{\phi}(n)$ for agents with a memory horizon of $\Theta(T^c)$, for any $c \in (0,1)$, where is $\phi$ is a function of other parameters (notably, the agent's exploration term $\lambda > 0$) which here can be arbitrarily small. This contrasts with theresults for $\EIRD$, where $\lambda$ must be bounded away from $0$ in order to ensure that the target set is non-empty.

In Section \ref{sec:hardness}, we show that determining the reward of the best distribution which is instantaneously realizable from itself, given a circuit which encodes an agent's preference model, is $\NP$-hard even for fixed rewards. 
We argue that this represents strong evidence that competing with any target set nontrivially larger than $\EIRD$ requires additional structural assumptions on preferences (such as our scale-bounded property).

Finally, in Section \ref{sec:smallgamma} we consider agents with an effective memory horizon of $O(1)$. Here we again highlight the usefulness of the scale-bounded property, and give an algorithm which obtains $o(T)$ regret with respect to $\Delta^{\phi}(n)$ whenever rewards do not change too frequently; in contrast, for general preferences, we show a regret lower bound over the $\EIRD$ set which is quasipolynomial in $n$, even when rewards are fixed (our algorithmic regret bounds are polynomial in all parameters).

Taken as a whole, in addition to providing several new algorithms for efficient regret minimization in online recommendations, our results highlight a number of delicate tradeoffs between the choice of target set, the speed at which an agent's memory is updated, and structural assumptions on preference functions, which may be informative for agent preference modeling more broadly.

\subsection{Related Work}

\paragraph{Stochastic bandits with changing rewards} 
Problems where future rewards are causally affected by actions have been studied in the stochastic multi-armed bandit 
setting \citep{gittins1979bandit, HeidariKR16,LevineCM17,KleinbergI18, DBLP:journals/corr/abs-1802-05693,LeqiKLM21, Laforgue+22, AwasthiBGK22, Papadigenopoulos+22}.
Most recent work on such problems has focused on specific models for reward evolution with motivations such as agent satiation, agent boredom, and congestion, and considers a single action being chosen in each round.
In contrast, our setting allows for adversarial rewards, and preference evolution is determined by interplay between the multiple items we choose in each menu.

\paragraph{Models of preference dynamics}
There has also been substantial work in understanding preference dynamics in recommendation systems, with an emphasis on linear update models.
\cite{Hazla+19, GaitondeKT21} consider a model for political preference dynamics where vector preferences 
drift towards the agreement or disagreement on randomly drawn issues, and 
\cite{DeanM22} study a similar model in the context of personalized recommendations for a single item.  A related model is also considered by \cite{jagadeesan2022supplyside} to study the influence of recommendations on genre formation. Further, interactions between payment incentives and self-reinforcing preferences are studied in a bandit setting by \cite{DBLP:journals/corr/abs-2105-08869}.

\paragraph{Reinforcement learning and recommendation systems}
There is prior work leveraging reinforcement learning for recommendations to maximize long-run rewards \citep{IeJWNAWCCB19, ZhanCLOMBBCC21, ChenBCJBC19}, typically with a focus on empirical evaluation.
There is also a long line of work on studying bias, feedback loops, and ``echo-chamber'' effects at the population level in recommendation systems \citep{flaxman,CurmeiHRH22,abdollahpouri2019managing,DBLP:journals/corr/abs-2007-13019}.

\paragraph{Dueling bandits}
The ``dueling bandits'' framework 
studies a recommendation problem similar to ours in which multiple items are chosen in each trial and 
relative feedback is observed, representing agent selections
\citep{10.1145/1553374.1553527, YUE20121538, NEURIPS2020_d5fcc35c, https://doi.org/10.48550/arxiv.2101.01572}.
However, in contrast to our setting, these works consider preference models which are fixed {\it a priori}, and do not change as a function of item history.

%\section{\MakeUppercase{Model and Preliminaries}}
\section{Preliminaries}
\label{sec:prelims}

Throughout, we use $\Delta(n)$ to denote the simplex over $n$ items, $\mathbf{u}_n$ for the uniform distribution on $n$ items, and $d_{TV}(v,v')$ for the total variation distance between distributions. We use the $\ell_2$ norm unless specified otherwise (e.g.\ as $\norm{x}_{1}$), and we use $B_{\epsilon}(x)$ to denote the ball of radius $\epsilon$ around $x$.

\subsection{Interaction Model}

We recall the setup from \cite{AgarwalB22} for the online recommendations problem for an agent with adaptive preferences. At any time, there is some {\it memory vector} $v \in \Delta(n)$, which expresses some function of the prior selections of the agent.
The {\it preference model} of an agent is a mapping $M : \Delta(n) \rightarrow [0,1]^n$ which assigns scores $M(v)_i = f_i(v)$ according to preference functions $f_i : \Delta(n) \rightarrow [0,1]$ for each item.   
An instance of the problem is specified by a universe of
$n$ items, a menu size $k < n$, a preference model $M$, a memory update rule $U$, and a sequence of reward vectors $r_1, \ldots, r_T$.
In each round $t \in [T]$:
\begin{itemize}
    \item the recommender chooses a menu $K_t$, consisting of $k$ distinct items from $[n]$, which is shown to the agent; 
    \item the agent selects one item $i_t \in K_t$, chosen at random according to the distribution given by:
    \begin{align*}
        p_t(i ; K_t, v_t) =&\; \frac{f_{i}(v_t)}{\sum_{j \in K_t} f_j(v_t) };
    \end{align*}
    \item the memory vector is updated to $v_{t+1} = U(v_t, i_t, t)$ by the update rule;
    \item the recommender receives reward $r_t(i_t) \in [0,1]$ for the chosen item.
\end{itemize}
The initial $v_1 \in \Delta(n)$ can be chosen arbitrarily. We assume each $f_i$ is unknown to the recommender, but $U$ is known. The goal of the recommender is to minimize regret over $T$ rounds with respect to some {\it target set} $S \subseteq \Delta(n)$. For any such $S$, the regret of an algorithm $\A$ with respect to $S$ is
\begin{align*}
    \Reg_S(\A; T) =&\;  \E\brackets{ \max_{x \in S}  \sum_{t=1}^T \langle r_t, x\rangle - r_t(i_t)}
\end{align*}
where $i_t$ is the agent's item choice at time $t$ resulting from $\A$, and where the expectation is taken over internal randomness of $\A$ as well as the agent's choices.

\subsection{Realizable Distributions}
For a preference model $M$ and memory vector $v$,
let $\IRD(v, M)$ denote the set of {\it instantaneously realizable distributions} for $v$, given by
\begin{align*}
    \IRD(v, M) =&\; \convhull \braces{  p(K, v) : K \in \brackets{ {n \choose k}}} 
\end{align*}
where $K$ is a $k$-item subset of $[n]$ and $p(K,v)$ denotes the item selection distribution of an agent with memory $v$ conditioned on being shown a menu $K$, which is given by
\begin{align*}
    p(i;K,v) =&\; \frac{ f_i(v) }{ \sum_{j\in K} f_j(v) }
\end{align*}
for each item $i$ in $K$ (and 0 otherwise), where $f_i$ is the preference scoring function for any item $i$. 
Note that this expresses the possible item choice distributions of an agent with memory $v$ resulting from all possible menu selection strategies by the recommender, as any distribution over menus yields a convex combination of item distributions $p(K,v)$.
The set of {\it everywhere instantaneously realizable distributions} is given by
\begin{align*}
    \EIRD(M) = \bigcap_{v\in\Delta(n)} \IRD(v, M).
\end{align*}
This is the target set considered by \cite{AgarwalB22}, and which we will consider for several of our results.
As a toy example, consider the case where preferences are fixed at $f_i(v) = 1$ for all $v$; here, every $\IRD$ set is equivalent to $\EIRD$, which is given by the convex hull of all vectors in $\Delta(n)$ with mass $1/k$ on exactly $k$ items. 
We note that there are several natural reasons to consider regret benchmarks in ``item space'' rather than ``menu space''; in addition to previously-shown linear regret lower bounds for the best fixed menu distribution, our rewards are determined by the items chosen by the agent rather than our recommendations, and choice distributions are not guaranteed to quickly stabilize even if we hold menu distributions fixed.
We will assume that scoring functions $f_i$ are in fact bounded in the range $[\lambda, 1]$ for some constant $\lambda > 0$ which captures exploration on behalf of the agent; in Section \ref{subsec:alg-biggamma-eird} we will assume $\lambda \geq {k}/{n}$, which ensures that $\EIRD$ is non-empty (and contains the uniform distribution), yet in Section \ref{subsec:alg-ss} we allow $\lambda$ to be arbitrarily small.  

\subsection{Discounted Memory Agents}
Throughout, 
we consider agents whose memory update rules are {\it$\gamma$-discounted}. 
\begin{definition}[Discounted Memory Updating]
Under the $\gamma$-discounted memory update rule $U_{\gamma}$, for some $\gamma \in [0,1]$, when an item $i_t$ is selected at round $t$, the memory vector $v_t$ is updated to $v^{t+1} = U_{\gamma}(v_t, i_t, t)$, with
\begin{align*}
    v_{t+1}(i) =&\; \frac{\sum_{s=1}^{t} \gamma^{t-s} \cdot \mathbf{1}(i = i_s)}{ \sum_{s=1}^{t} \gamma^{t-s}}.
\end{align*}
\end{definition}
Taking $\gamma = 1$ yields the  uniform memory update rule considered by \cite{AgarwalB22}.
As in many settings involving discount factors, we can think of values of $\gamma$ closer to 1 as corresponding to larger ``effective horizons'' for memory; for any $\gamma \leq 1 - o(1)$ we refer to the quantity $1/(1- \gamma)$ as simply the {\it effective memory horizon}. We assume throughout that $\gamma$ is known.

\subsection{Smooth Preference Models}

Many of our results consider preference models with the property that each scoring function 
is Lipschitz, in addition to being bounded above 0, which we refer to as {\it smooth} preference models. 
\begin{definition}[Smooth Preference Models]
A preference model $M$ is $(\lambda, L)$-smooth if each scoring function $f_i$ takes values in $[\lambda, 1]$ and
is $L$-Lipschitz over $\Delta(n)$ with respect to the $\ell_1$ norm.
\end{definition}

This property allows for quite a broad class of functions, and is satisfied by each of the classes in \cite{AgarwalB22} (e.g.\ low-degree polynomials) with appropriate parameters. 
Despite its generality, we show in Section \ref{sec:biggamma} that this assumption alone is sufficient to enable us to always maintain an accurate {\it local} approximation of the model, provided that the agent's memory vector does not change too quickly, by periodically implementing a query learning rountine. For convenience, we assume that preference scores are always normalized to have a constant sum $\sum_i f_i(v) = C$  for some $C$ and any $v$, yet this can be relaxed for each of our results up to $\poly(n,L)$ factors.

\section{Targeting $\EIRD$ for Agents with Long Memory Horizons}
\label{sec:biggamma}

We begin by considering cases where $\gamma = 1 - o(1)$, with an effective memory horizon of $\Omega(T^c)$ for some $c \in (0,1]$.
Here, memory vectors change slowly, and every point in $\Delta(n)$ is well-approximated {\it some} sequence of item selections.
In Section \ref{subsec:menutimes} 
we give a result on the structure of $\IRD$ sets, which enables efficient menu selection in our $o(T)$-regret algorithm for $\EIRD$ in \ref{subsec:alg-biggamma-eird}.

\subsection{Characterizing $\IRD$ via Menu Times}
\label{subsec:menutimes}
We introduce a notion of the {\it menu time} required by each item in order to induce a particular item distribution $x$,
which allows us to directly characterize $\IRD$ sets, as well as efficiently construct menu distributions at each round with a greedy approach, avoiding the exponential dependence on $k$ from the linear programming routine in \cite{AgarwalB22} which enumerates all ${n \choose k}$ menus. At a memory vector $v$, for a target item distribution $x$, the menu time for item $i$ is given by
\begin{align*}
    \mu_i =&\; \frac{k \cdot \frac{x_i}{f_i(v)}}{\sum_{j=1}^n \frac{x_j}{f_j(v)}}.
\end{align*}
Observe that these quantities always satisfy 
$\sum_i \mu_i = k$ for any $v$ and $x$. 
We show that a distribution $x$ can be realized from a memory vector $v$ if and only if $\max_i \mu_i \leq 1$.

\begin{lemma}\label{lemma:ird-menus}
An item distribution $x$ belongs to $\IRD(v, M)$ if and only if we have that the menu time $\mu_i$ for each item is at most $1$. If this condition holds, there is a $\poly(n)$ time algorithm \textup{\texttt{MenuDist}$(v, x, M)$} for constructing a menu distribution $z$ such that $\E_{K\sim z}[p(K,v)] = x$.
\end{lemma}

\begin{proofsketch}
    If $x \in \IRD(v, M)$, there exists some menu distribution $z$ which yields $x$; converting this menu distribution to $\mu_i$ values by ``crediting'' a menu in proportion with its mass and the inverse of the sum of its item scores results in a menu time vector satisfying $\sum_i \mu_i = k $ and $\max_i {\mu_i} \leq 1$.

    Given a menu time vector satisfying these conditions, we can construct such a distribution by greedily choosing a menu of the $k$ items with highest remaining menu time and ``charging'' their remaining menu times at the same rate, breaking ties for the $k$th highest by charging and including at fractional rates. The number of items tied for $k$th highest remaining $\mu_i$ increases by 1 at each stage, and the highest initial $k-1$ items (with $\mu_i \leq 1$) will be included non-fractionally until tied for $k$th highest.
    The mass of each added menu in our final distribution $z$ will be allocated proportionally to the sum of scores of items in the menu. This allows cancellation of the terms for sums of menu scores, resulting in a menu distribution where the selection probability of an item is proportional to its score $f_i(v)$ and the number of (fractional) stages in which it was added to the menu. As the latter number of stages in which an item is added to a menu is proportional to its menu time, and its menu time is proportional to $x_i / f_i(v)$, the induced item choice distribution is then proportional to $x_i$.
\end{proofsketch}
\texttt{MenuDist}$(v, x, M)$ directly implements this menu distribution construction, and is used by our algorithms in Section \ref{subsec:alg-biggamma-eird} and Section \ref{subsec:alg-ss}. Further details are given in Appendix \ref{sec:menutime-proof}.

\subsection{A Useful Algorithm for Adversarial Bandits}
\label{subsec:dbg}

Here, we introduce a new algorithm for adversarial bandit problems with a number of useful robustness properties, which serves as the centerpiece of our approach in Section \ref{subsec:alg-biggamma-eird}. Our algorithm, Deferred Bandit Gradient, can tolerate unobserved adversarial perturbations $\xi_t$ to the action distribution $x_t$ in each round (where $x_t$ is corrupted to $y_t = x_t + \xi_t$ prior to sampling), and can accommodate {\it contracting} decision sets (where the action distribution $x_t$ chosen in each round must lie in a set $\K_t$, with $\K_t \subseteq \K_{t-1}$). Both of these properties were identified as being useful in this setting by \cite{AgarwalB22}, as preference scoring estimates will inevitably have some imprecision, and we generally will not know the shape of the $\EIRD$ set in advance. For online optimization in general, the  contracting domains property appears challenging to obtain with algorithms resembling Follow the Regularized Leader (such as Hedge or EXP3), yet is much more straightforward with approaches resembling (projected) Online Gradient Descent. In contrast to \cite{AgarwalB22}, who extend the ``OGD-style'' FKM algorithm for bandit convex optimization which obtains $O(T^{3/4})$ regret (\cite{DBLP:journals/corr/cs-LG-0408007}), our algorithm operates directly in the adversarial bandit setting for linear losses over the simplex, and leverages linearity to decrease the variance in gradient estimates (which bottlenecks the regret of FKM) by ``deferring'' the contribution of each reward observation across several future rounds, enabling runtime improvements. 
\begin{algorithm}[h]
\caption{Deferred Bandit Gradient}
\begin{algorithmic}
\STATE Input: sequence of rewards $r_t$, perturbation vectors $\xi_1,\ldots , \xi_T$ where $\abs{\xi_{t, i}} \leq \frac{ \epsilon }{n} $ at each round $t$, and contracting convex decision sets $\K_1, \ldots \K_T$ where $\B_{\epsilon} \subseteq \K_T$ for a given $\epsilon$.
\STATE Set $x_1 = \mathbf{u}_n$, $H = \frac{n}{\epsilon}$
\FOR{$t = 1$ to $T$}
    \STATE Adversary perturbs distribution $ {x}_t$ to $y_t = {x}_t + \xi_t$
    \STATE Sample $i_t \sim y_t$, observe $i_t$ and reward $r_t(i_t)$ 
    \STATE Let $\tilde{r}_{t} = \frac{e_{i_t} }{H} \cdot \frac{r_{t,i_t} }{ {x}_{t,i_t}}$ and $\widetilde{\nabla}_t = \sum_{s = \max(t - H + 1, 1)}^{t} \frac{\tilde{r}_t }{H}$
    \STATE Let $\K_{t+1, \epsilon} =\{x \vert \mathbf{u}_n + \frac{1}{1 - \epsilon} (x - \mathbf{u}_n) \in \K_{t+1} \}$
    \STATE Update $x_{t+1} = \Pi_{\K_{t+1,\epsilon }}[x_t + \eta \widetilde{\nabla}_t]$ 
\ENDFOR
\end{algorithmic}
\end{algorithm}

\begin{theorem}\label{thm:DBG}
For a sequence of rewards $r_t, \ldots, r_T \in [0,1]^n$, contracting convex decision sets $\K_1, \ldots \K_T \subseteq \Delta(n)$ where $\mathbf{u}_n \in \K_T$, 
and perturbation vectors $\xi_1,\ldots , \xi_T$ satisfying $\abs{\xi_{t, i}} \leq \frac{ \epsilon }{n}$ 
for a given $\epsilon$ in each round $t$, 
Deferred Bandit Gradient obtains expected regret bounded by
\begin{align*}
    \max_{x^* \in \K_t} \sum_{t=1}^T r_t^{\top}x^* - \sum_{t=1}^T r_t^{\top}y_t \leq&\; 2 \eta n^2 T  + \frac{\sqrt{2}}{\eta} +  3\sqrt{n} \epsilon T + \frac{n}{2\epsilon} + \sum_{t=1}^T \sum_{i=1}^n \frac{\abs { \xi_{t,i} }}{x_{t,i}}. 
\end{align*}
\end{theorem}
We prove Theorem \ref{thm:DBG} in Appendix \ref{appendix:dbg}. Our analysis proceeds by tracking the regret of a variant of OGD which accommodates contraction over the expectations of the sequence of $\widetilde{\nabla}_t$ vectors, whose squared norms are small in expectation, and showing that this closely tracks both the regret obtained by our algorithm and the reward of the optimum $x^*$ in hindsight.  Note that if the constraint $\abs{\xi_{t, i}} \leq \frac{ \epsilon x_{t,i} }{n} $ is satisfied in each round, DBG can be calibrated to obtain regret $O(n\sqrt{T} + \epsilon\sqrt{n}T)$.

\subsection{Targeting EIRD}
\label{subsec:alg-biggamma-eird}

At a high level, our approach is to guide the agent to implicitly run Deferred Bandit Gradient on our behalf, over a contracting subset of $\Delta(n)$ which always contains $\EIRD$. 
Periodically, we pause in order to refresh our estimates of the agent's current preferences, wherein all items are shown to the agent sufficiently often in order to accurately estimate preference scores near the current memory vector. We leverage smoothness of preferences to determine accuracy bounds on our score estimates as $v_t$ updates.
Given a target item distribution $x_t$ for the agent to sample from (as selected by $\dbg$), we can invoke the \texttt{MenuDist} routine with our score estimates to construct a menu distribution $z_t$ which approximately induces a choice distribution of $x_t$ (whose error is represented by the perturbations $\xi_t$ for $\dbg$). 

As our approach relies on stability of memory vectors across rounds, our regret decays towards $\Theta(T)$ as the memory horizon vanishes relative to $T$; we discuss the challenges associated with short memory horizons further in Section \ref{sec:smallgamma}. For memory horizons of $T^c$ for any $c >0$ we obtain strictly sublinear regret, and we recover the $\tilde{O}(T^{3/4})$ rate from \cite{AgarwalB22} for uniform memory, now holding for any smooth preferences rather than only for specific parametric classes.

\begin{theorem}\label{thm:long-eird-alg}
For an agent with a $(\lambda, L)$-smooth preference model $M$ for $\lambda \geq k/n$, and $\gamma$-discounted memory for $\gamma \geq 1 - \frac{1}{T^c}$ and $c \in (0,1]$, Algorithm \ref{alg:long-eird} obtains regret bounded by
\begin{align*}
\max_{x^* \in \EIRD(M)} \sum_{t=1}^T r^{\top}_t x^* - \E\brackets{\sum_{t=1}^T r_t(i_t) }
    =&\; \tilde{O}\parens{(n/\lambda)^{3/2} L^{1/4} \cdot T^{1 - c/4} }.
\end{align*}
\end{theorem}
\begin{algorithm}
    \caption{(Targeting $\EIRD$ for Smooth Models).}\label{alg:long-eird}
    \begin{algorithmic}
        \STATE Let $c^* = \min(c, 3/4)$, $\epsilon = \tilde{O}(nL^{1/4} \lambda^{-3/2} T^{-c/4})$, $Q = \tilde{O}(\frac{n^2}{\lambda^4\epsilon^2})$, and $\eta = (nT)^{-1/2}$.
        \STATE Initialize $q = 0$, $v_0 = v^* = \mathbf{u}_n$,  $F_{i} = \frac{C}{n}$ for $i\in [n]$, $M^* = \{F_i\}$
        \STATE Initialize $\dbg$ for $\epsilon, \eta$.
        \WHILE{$t\leq T$}
            \IF{$t < T^{c^*}$}
            \STATE Show arbitrary menu $K_t$ to agent
            \ELSIF{$t \geq T^{c^*}$ and either $\norm{ v_t - v^* }_1 \geq \frac{\lambda\epsilon}{2nL}$ or $q = 0$}
            \FOR{$b \in \{0,\ldots ,\lceil \frac{n-1}{k-1} \rceil\}$}
            \STATE Show agent menu $K_b = \{1\} \cup \{b(k-1)+2,\ldots,(b+1)(k+1) + 1\}$ for $Q$ rounds
            \STATE Let $\hat{F}_i = \parens{ \text{\# times } i_t = i} / \parens{ \text{\# times }i_t = 1}$ within the $Q$ rounds, for $i \in K_b$
            \ENDFOR
            \STATE Set $F_i = \frac{C \cdot \hat{F}_i}{\sum_{j=1}^n \hat{F}_j}$ for each $i\in[n]$, $M^* = \{F_i\}$
            \STATE Set $v^* = v^t$, increment $q$ by $\lceil \frac{n-1}{k-1} \rceil \cdot Q$, set $\K_{t-q+1} = \K_{t - q} \cap \IRD(v^*, M^*)$ for $\dbg$
            \ELSE
            \STATE Get $x_{t-q}$ from $\dbg$
            \STATE Let $z_t = \texttt{MenuDist}(v_t, x_{t-q}, M^*)$, sample menu $K_t \sim z_t$
            \STATE Show $K_t$ to agent, update $\dbg$ with observed $i_t$ and $r_t(i_t)$
            \ENDIF 
            \STATE Set $v_{t+1} = U(v_t, i_t, t)$, for each round counted by $q$ if necessary
        \ENDWHILE
    \end{algorithmic}
\end{algorithm}

The complete proof of Theorem \ref{thm:long-eird-alg} is deferred to Appendix \ref{thm:long-eird-alg}. Building on the regret bound for $\dbg$ in Theorem \ref{thm:DBG}, the central challenges are to show that preference estimates $F_i$ are close enough to each $f_i(v_t)$ to enable accurate choice targeting via \texttt{MenuDist}, and that the number of non-$\dbg$ rounds spent updating $F_i$ (tracked by $q$) does not grow too quickly. While $\EIRD$ contains at least the uniform distribution (as implied by Lemma \ref{lemma:ird-menus} when $\lambda \geq \frac{k}{n}$), it may not be particularly large in general.
In Section \ref{sec:biggamma-pi}, we identify conditions under which an alternate algorithmic approach allows us to compete with a much larger set of item distributions than $\EIRD$.

\section{Beyond EIRD: Scale-Bounded Preferences and the Smoothed Simplex}\label{sec:biggamma-pi}
\label{subsec:pseudoinc}

One motivation given by \cite{AgarwalB22} for considering  $\EIRD$  is the difficulty of exploration under uniform memory, as the current memory cannot be repeatedly ``washed away'' without exponential blowup. However, considering discount factors of $\gamma  \leq 1 - o(1)$ introduces the possibility that we might be able efficiently explore the space of feasible vectors and compete against item distributions which lie outside of $\EIRD$, i.e.\ item distributions which are only feasible for a strict subset of all memory vectors.  We identify a structural property which enables this, wherein preference scoring function outputs cannot be too far multiplicatively from their item's weight in memory. We say that such functions are {\it scale-bounded}.

\begin{definition}[Scale-Bounded Functions] 
A preference scoring function $f_i : \Delta(n) \rightarrow [\frac{\lambda}{\sigma},1]$ is $(\sigma,\lambda)$-scale-bounded for $\sigma \geq 1$ and $\lambda > 0$ if
\begin{align*}
    \sigma^{-1}((1 - \lambda)v_i + \lambda) \leq&\; f_i(v) \leq \sigma ((1 - \lambda)v_i + \lambda).
\end{align*}
\end{definition}
We say that a preference model $M$ is scale-bounded if each $f_i$ is scale-bounded; in this case, the vector of scores $M(v)$ cannot stray too far from their values in $v$. 
When this property is satisfied, we can show (using the menu time approach from Lemma \ref{lemma:ird-menus}) that any point which is not too close to the boundary of $\Delta(n)$ is contained in its own $\IRD$ set.
This motivates a target set of all such points, 
which we term the {\it $\phi$-smoothed simplex}.
\begin{definition}[$\phi$-Smoothed Simplex]
For any $\phi \in [0,1]$, the $\phi$-smoothed simplex is the set given by $\Delta^{\phi}(n) = \{ (1 - \phi)x + \phi \mathbf{u}_n : x \in \Delta(n) \}$.
\end{definition}
Further, we show that a neighborhood around any such point is contained in $\Delta^{\phi}(n)$ as well; as preference scores cannot be too far from an item's current memory vector weight, the required menu time for any item in a distribution $x$ which is nearby $v$ cannot be too large.
\begin{lemma}\label{lemma:pseudoinc-ird}
Let $M$ be a $(\sigma, \lambda)$-scale-bounded 
preference model
with $\sigma \leq \sqrt{n/(2k)}$.
Then, $x \in \IRD(v, M)$
for any $x \in  B_{\lambda \phi}(v)  \cap \Delta^{\phi}(n)$
and any $v \in \Delta^{\phi}(n)$, provided that $\phi \geq {4 k\lambda \sigma^2}$. 
\end{lemma}
Here, we consider $\Delta^{\phi}(n)$ as a target set for regret minimization. 
We will no longer require explicit lower bounds on $\lambda$, and so we can take our regret benchmark to be approaching the entire simplex as $\lambda$ approaches 0 with an appropriate choice of $\phi$.
This presents a stark constrast with the $\EIRD$ benchmark, as the scale-bounded property now suggests that it may be possible to persuade the agent to pick the best item in nearly every round, rather than in $O(T/n)$ rounds (which may occur in Section \ref{subsec:alg-biggamma-eird}, e.g.\ if some $f_i(v) = \lambda = k/n$ at every $v \in \Delta(n)$).

\subsection{A No-Regret Algorithm for $\Delta^{\phi}(n)$}\label{subsec:alg-ss}

In contrast to Algorithm \ref{alg:long-eird}, where we considered each round as a step for a bandit optimization algorithm with interleaved learning stages, 
here we collapse multiple iterations of learning and targeting into a {\it single} step for Online Gradient Descent, run over $\Delta^{\phi}(n)$,
where the agent's entire memory vector is moved in a desired direction. 
While we can no longer instantaneously realize any distribution in our target set, the ability to induce any choice distribution in a nearby ball enables exploration throughout $\Delta^{\phi}(n)$ via the agent's memory vector.
Further, the scale-bounded condition tethers scores to memory weight, enabling reduced variance in estimating both rewards and preferences.
However, this also yields a decay as $c$ approaches 1 (in addition to 0), as memory does not update quickly enough to enable exploration. Theorem 1 from \cite{AgarwalB22} implies that this is necessary: an adversary may shift the reward distribution in later rounds when we can no longer significantly move the entire memory vector, necessitating linear regret.

\begin{theorem}[Scale-Bounded Discounted Regret Bound]\label{thm:long-ss-alg}
For any agent with a preference model $M$ which is $(\sigma, \lambda)$-scale-bounded and $(\frac{ \lambda}{ \sigma}, L)$-smooth with $\sigma \leq \sqrt{n/(2k)}$, and with $\gamma$-discounted memory for $\gamma = 1 - \frac{1}{T^c}$ for $c \in (0,1)$, Algorithm \ref{alg:long-ss} obtains 
regret
\begin{align*}
    \Reg_{\Delta^{\phi}(n)}(\A_2; T) =&\; \tilde{O}\parens{(n^{4}L \parens{ T^{1 - c/2} + T^{1/2 + c/2}  }}
\end{align*}
with respect to the $\phi$-smoothed simplex, for $\phi = {4k\lambda \sigma^2}$.
\end{theorem}

\begin{algorithm}
    \caption{(Targeting $\Delta^{\phi}$ for Scale-Bounded Models).}\label{alg:long-ss}
    \begin{algorithmic}
        \STATE Let $\epsilon = \tilde{O}(n^4L \cdot \max( T^{-c/2}, T^{c/2 - 1/2} ))$, $S = \tilde{O}(n^{3/2} T^c )$,  $\eta = \tilde{O}(n^{-1/2}L \cdot T^{c/2 - 1/2})$.
        \STATE \texttt{--- burn-in ---}
        \FOR{$t = 1$ to $T^c$}
            \STATE Show agent menu $K = \{1,\ldots,k\}$
        \ENDFOR
        \WHILE{$\max_i \abs{v_{t,i} - \frac{1}{n}} \geq \frac{\epsilon}{4n^2L\sigma}$ }
            \STATE Show agent $k$ items with smallest $v_{t,i}$, choosing randomly among ties up to $T^{-c}$
        \ENDWHILE
        \STATE \texttt{--- initial learning ---}
        \FOR{$T^c$ rounds}
            \STATE Let $F_{t,i} = \sigma^{-1}((1 - \lambda)v_{t,i} + \lambda)$ if $v_{t,i} < \frac{1}{n}$, else $F_{t,i} = \sigma ((1 - \lambda)v_{t,i} + \lambda)$
            \STATE Let $z_t = \texttt{MenuDist}(v_t, \mathbf{u}_n,  \{F_{t,i}\})$, show agent $K_t \sim z_t$
        \ENDFOR
        \STATE Set $F_{i}=\sum_{t } \mathbf{1}(i_t=i) \cdot F_{t,i} / (\frac{1}{n} T^c)) \cdot C(\sum_{j=1}^n F_{t,j})^{-1}$, set $v^* = v_t$
        \STATE \texttt{--- optimization ---}
        \STATE Initialize $\ogd$ over $\Delta^{\phi}_{\epsilon}(n)$ for $T/S$ rounds with $\eta$, set  $x_1 = v^* := v_t$.
        \FOR{$s = 1$ to $T/S$}
            \STATE Receive $x_s$ from $\ogd$
            \FOR{$S$ rounds do}
            \IF{$\norm{v_t - v^*}_1 \geq \epsilon / (2nL)$}
            \STATE Let $\tilde{v} = v_t$
            \STATE Show agent $K_t \sim z = \texttt{MenuDist}(v_{t}, \tilde{v}, \{F_i\}) $ for $T^c / L^2$ rounds
            \STATE Set $F_i = \sum_{t} \mathbf{1}(i_t=i) \cdot F_i / ( \tilde{v}_i T^c /L^2 ) \cdot C(\sum_{j=1}^n F_{j})^{-1}$, set $v^* = v_t$
            \ENDIF
            \STATE Show agent $K_t \sim z_t = \texttt{MenuDist}(v^*, x_s, \{F_i\}) $
            \ENDFOR
            \STATE Set $\tilde{\nabla}_s = \sum_{h = t-S + 1}^t  e_{i_h} r_h(i_h) /( x_{h,i} S ) $, update $\ogd$
        \ENDFOR
    \end{algorithmic}
\end{algorithm}

\begin{proofsketch}
In the ``burn-in'' stage, our goal is simply to push the memory vector towards $\mathbf{u}_n$; by first saturating memory on only $k$ items for $T^c$ rounds, we are then able to push all low-memory items towards $\frac{1}{n}$ at near-uniform rates, as memory now moves slowly and the $k$ lowest values will remain close together by the scale-bounded condition, so no item can get ``stuck'' near 0 and we reach $\mathbf{u}_n$ in $\tilde{O}(T^c)$ rounds. In the ``initial learning'' stage, we are now promised that the uniform distribution is is within our $\IRD$ set, and we can force memory to remain there by assuming  pessimistic scores if $v_{t,i} < \frac{1}{n}$ and  optimistic scores otherwise. By comparing observed selection frequencies to those indicated by our assumed scores, we obtain unbiased estimators for the true $f_i$ values near $\mathbf{u}_n$. In the ``optimization'' stage, we batch $O(T^{c})$ rounds into ``steps'' for Online Gradient Descent large enough to maintain locally accurate $f_i$ estimates throughout, and alternate between progressing toward the chosen target and updating our scores, which further enables concentrated estimates of average reward vectors in each step and a regret bound akin to that for ``slowed down'' OGD.
\end{proofsketch}

We allow $\lambda$ to be arbitrarily small, and assume only that $T$ is large enough to yield $\lambda \geq T^{-c/4}\poly(n)$; our bound has no dependence on $\lambda$ or $\phi$ beyond this. Our optimal rate over $c$ is again $\tilde{O}(T^{3/4})$, yet this time occurring when $c=\frac{1}{2}$, balancing improved variance reduction in learning with the need to quickly explore in memory space. Full proofs for this section are contained in Appendix \ref{sec:proofs-ss}.

\section{On Hardness of Relaxing Benchmarks}\label{sec:hardness}

Here, we give a hardness result relating to the complexity of determining the optimal reward of a preference model for a fixed loss function, which we view as evidence that expanding to target sets beyond $\EIRD$ necessitates carefully tailored assumptions on preference models, such as the scale-bounded property we consider in Section \ref{sec:biggamma-pi}.

\begin{theorem}\label{thm:np-hardness}
Unless $\textup{\textsc{RP}} \supseteq \textup{\textsc{NP}}$, there is no polynomial time algorithm 
which takes as input a circuit representation of a preference model $M$ and linear reward function $r$, and approximates the reward $r(v)$ of the best distribution $v \in \IRD(v, M)$ contained in its own $\IRD$ set within a $O(1/n)$ factor.
\end{theorem}
We show that an instance of the ``Max Independent Set'' problem, which is \textup{\textsc{NP}}-hard to approximate, can be encoded in a  preference model for an agent, wherein optimizing reward corresponds to selecting any independent set. We suppose there is only 1 item which receives positive reward and will always be in the menu, and the objective corresponds to maximizing its score. Our construction operates by interpreting assignments of weight to items as proposals for possible independent sets, and then efficiently checks a graph for edges between the corresponding vertices; the score of item 1 is then proportional to the size of any valid independent set.

This suggests that the difficulty of optimization beyond $\EIRD$ stems not only from issues of learnability or adversarial losses, but rather the possible complexity of optimal strategies which can be encoded by a preference model. In general, it appears hopeless to attempt to compete with a distribution which is {\it not} in its own $\IRD$ set, and any point which {\it is} inside its $\IRD$ set is {\it stable} under long enough time horizons for smooth preference models (up to arbitrary approximation) provided it can be initially reached; as such, properties similar to Lemma \ref{lemma:pseudoinc-ird} appear necessary for identifying feasible targets. 

\section{Agents with Short Memory Horizons}
\label{sec:smallgamma}

When the discount factor of the agent is small enough that memory vectors may move rapidly, we lose the precision required by the algorithms in Section \ref{sec:biggamma} in order to implement queries, and in fact the feasible state space may more closely resemble a discrete grid, with memory vectors encoding the sequence of items chosen over an effective horizon which is constant with respect to $T$. Nonetheless, for scale-bounded models we give an algorithm which we call \textsc{EXP}-$\phi$, which obtains $o(T)$ regret with respect to $\Delta^{\phi}(n)$ for {\it any} value of $\gamma \in [0,1)$ under an assumption about the restricted adversarial nature of rewards. Here, we assume that rewards are stochastic rather than adversarial for windows of length $o(T)$, but distributions may change adversarially between each window; we require a slightly larger lower bound on $\phi$ (yet still  $O(\lambda)$).

\subsection{A No-Regret Algorithm for Scale-Bounded Models}\label{subsec:exp-phi}
The idea behind \textsc{EXP}-$\phi$ is to view each vertex of the smoothed simplex as an action for a multi-armed bandit problem, where each ``pull'' corresponds to several rounds. When we ``commit'' to playing an item in the menu for a sufficiently long time, while otherwise playing items with the smallest weight in memory, the scale-bounded property will ensure that the selection frequency of that item gravitates towards its vertex in the smoothed simplex. Further, as we are no longer attempting to learn the preference model explicitly, we can relax the smoothness requirement.

\begin{theorem}
\label{thm:short-ss-alg}
For any agent with a preference model $M$ which is $(\sigma, \lambda)$-scale-bounded for which each $f_i(v) \in [\lambda, 1]$ for $\lambda \geq \frac{\sigma^2 k}{n}$ and $\sigma \leq \sqrt{n/(2k)}$, and with $\gamma$-discounted memory for $\gamma \in [0,1)$, when losses are drawn from a distribution which changes at most once every $t_{\text{hold}} = \tilde{O}\parens{ \frac{T^{2/3}}{1-\gamma} }$ rounds, Algorithm \ref{alg:long-ss} obtains 
regret at most
\begin{align*}
    \Reg_{\Delta^{\phi}(n)}(\A_3 ; T) =&\; \tilde{O}(T^{5/6})
\end{align*}
with respect to $\Delta^{\phi}(n)$, for $\phi = 3 \lambda k^3 \sigma^6$.  
\end{theorem}

\begin{algorithm}
    \caption{(\textsc{EXP}-$\phi$).}\label{alg:short-ss}
    \begin{algorithmic}[0]
        \STATE Initialize $\textsc{EXP3}$ to run for $T/t_{\text{hold}}$ steps %
        \WHILE{$t < T$}
            \STATE Sample arm $i^*$ from $\textsc{EXP3}$
            \FOR{$\tilde{O}(T^{2/3}/(1 - \gamma))$ rounds}
                \STATE Let $K_t = \{i^*\} + \text{argmin}^{k-1}_{j \neq i} v_j$ 
            \ENDFOR
            \STATE Update $\textsc{EXP3}$ with average reward of $i^*$
        \ENDWHILE
    \end{algorithmic}
\end{algorithm}

\subsection{Barriers for General Models}

If we cannot assume that preferences are scale-bounded, then it appears difficult to compete even against $\EIRD$ for arbitrary smooth models. We show a regret lower bound with respect to $\EIRD$ for any algorithm over a quasipolynomial time horizon by constructing preference models in which the optimal strategy depends delicately on the current memory vector, and which simultaneously induces fast exploration over a discrete state space.

\begin{theorem}\label{thm:short-eird-hard}
For any $\gamma \in (0,1/2)$, there is a set of 
$(\lambda, L)$-
smooth preference models $\mathcal{M}$ with $\lambda = O(1/n)$ and $L = \poly(n)$
such that any algorithm must have expected regret $\Omega(T)$ for any $T \in O(n^{\log (n)})$
when the preference model $M$ is sampled uniformly from $\mathcal{M}$. 
\end{theorem}
Our approach is to observe that every feasible memory vector encodes a unique truncated history of length $O(\log n)$, resulting in an implicit state space of size $O(n^{\log (n)})$. We design preference models in which the optimal policy is implementable by inducing the uniform distribution at each round, which lies inside $\EIRD$, yet requires identifying a specific set of alternate items to place in the menu deterministically at each state to maximize the selection probability of item 1. We show that any competitive strategy also necessarily explores many states with high probability, and so any algorithm will frequently reach states where it cannot identify the optimal menu distribution, which is defined on a per-state basis by a random process.

%\clearpage
\bibliography{ref}

\clearpage
\appendix
\section{Analysis for Deferred Bandit Gradient}
\label{appendix:dbg}
We first show that the analysis of vanilla Online Gradient Descent extends directly to adversarially contracting domains, where our chosen action $x_t$ must lie in the observed set $\K_t$ in each round.
\begin{algorithm}
\caption{Contracting Online Gradient Descent.}\label{alg:ogd-contracting}
\begin{algorithmic}
\STATE Input: sequence of contracting convex decision sets $\K_1, \ldots \K_T$, $x_1 \in \K_1$, step size $\eta$
\STATE Set $x_1 = \mathbf{0}$
\FOR{$t = 1$ to $T$}
    \STATE Play $x_t$ and observe cost $\ell_t(x_t)$
    \STATE Update and project: \begin{align*}
        y_{t+1} =&\; x_t - \eta \nabla \ell_t(x_t) \\
        x_{t+1} =&\; \Pi_{\K_{t+1}}(y_{t+1})
    \end{align*}
\ENDFOR
\end{algorithmic}
\end{algorithm}

\begin{lemma} \label{lemma:ogd-contracting}
For a sequence of contracting convex decision sets $\K_1, \ldots \K_T$, $x_1 \in \K_1$ each with diameter at most $D$, a sequence of $G$-Lipschitz losses $\ell_1,\ldots, \ell_T$, and parameter $\eta$,
the regret of Algorithm \ref{alg:ogd-contracting} with respect to $\K_t$ is bounded by
\begin{align*}
\sum_{t=1}^T \ell_t(x_t) - \min_{x^* \in \K_T }\sum_{t=1}^T \ell_t(x^*) \leq&\;  \frac{D^2}{2\eta} + \frac{\eta }{2} \sum_{t=1}^T \norm{\nabla_t}^2 \leq~ GD\sqrt{T}
\end{align*}
when $\eta = \frac{D}{G\sqrt{T}}$.
\end{lemma}

\begin{proof}
Let $x^* = \text{arg min}_{x \in \K_T} \sum_{t=1}^T \ell_t(x)$, and let $\nabla_t = \nabla \ell_t(x_t)$. First, note that
\begin{align*}
    \ell_t(x_t) - \ell_t(x^*) \leq&\; \nabla_t^{\top} (x_t - x^*)
\end{align*}
by convexity; we can then upper-bound each point's distance from $x^*$ by:
\begin{align*}
    \norm{x_{t+1} - x^*} =&\; \norm{ \Pi_{\K_{t+1}}(x_t - \eta \nabla \ell_t(x_t)) - x^*}  \leq \norm{x_t - \eta \nabla_t - x^*},
\end{align*}
as $x^* \in \K_{t+1} \supseteq \K_T$ . Then we have
\begin{align*}
    \norm{x_{t+1} - x^*}^2 \leq&\; \norm{x_{t} - x^*}^2 + \eta^2 \norm{ \nabla_t }^2 - 2\eta \nabla_t^{\top} (x_t - x^*)
\end{align*} 
and 
\begin{align*}
    \nabla_t^{\top} (x_t - x^*) \leq&\; \frac{\norm{x_t - x^*}^2 - \norm{x_{t+1} - x^*}^2 }{2 \eta} + \frac{\eta \norm{\nabla_t}^2 }{2}.
\end{align*}
We can then conclude:
\begin{align*}
    \sum_{t=1}^T \ell_t(x_t) - \sum_{t=1}^T \ell_t(x^*) \leq&\; \sum_{t=1}^T \nabla_t^{\top} (x_t - x^*) \\ 
    \leq&\; \sum_{t=1}^T \frac{\norm{x_t - x^*}^2 - \norm{x_{t+1} - x^*}^2 }{2 \eta} + \frac{\eta  }{2} \sum_{t=1}^T \norm{\nabla_t}^2 \\ 
    \leq&\; \frac{\norm{x_T - x^*}^2}{2\eta} + \frac{\eta }{2} \sum_{t=1}^T \norm{\nabla_t}^2  \\
    \leq&\; \frac{D^2}{2\eta} + \frac{\eta }{2} \sum_{t=1}^T \norm{\nabla_t}^2 \\
    =&\; GD\sqrt{T}.  \tag{ $\eta = \frac{D}{G\sqrt{T}}$}
\end{align*}
\end{proof}

\begin{proof} Equipped with the previous result, we can now prove the regret bound for Theorem \ref{thm:DBG}.
Let  ${r}^*_t = \sum_{s = \max(t - H + 1, 1)}^{t} \frac{ r_{s} \otimes (y_s \oslash x_s) }{H}$, where $\otimes$ and  $\oslash$ denote elementwise multiplication and division, respectively, and let $\hat{r}_t = \sum_{s = \max(t - H + 1, 1)}^{t} \frac{ r_{s} }{H}$. Further, let $x^*_{\epsilon} = \Pi_{\K_T, \epsilon}[x^*]$.
Observe that the following hold for every $t$:
\begin{align}
 \frac{r_{t,i} \cdot y_{t, i}}{x_{t, i}} =&\; r_{t,i} \parens{ 1 + \frac{\xi_{t,i}}{x_{t,i}} };  \notag \\
 \E[r^*_t  - \hat{r}_t ] =&\; \frac{1}{H}\sum_{s=t-H+1}^t ((r_s \otimes \xi_s) \oslash x_s) ; \label{eq:seq-step-deviation}  \\
    \E\brackets{ \widetilde{\nabla}_t } =&\; r^*_t. \label{eq:exp-seq-match}  
\end{align}
Observe that by the constraints on each $\K_{t, \epsilon}$, each $x_t$ can be expressed as $(1 - \epsilon)x + \epsilon \mathbf{u}_n$ for some $x \in \K{t} \subseteq \Delta(n)$, and so we will always have $x_{t,i} \geq \frac{\epsilon}{n}$.
To bound the squared norms of $\widetilde{\nabla}_t$ in order apply Lemma \ref{alg:ogd-contracting}, consider the maximizing case where $x_t = \frac{\epsilon}{n}$ and $y_t = \frac{2\epsilon}{n}$ in all but one element, and where $r_{t,i} = 1$ for all rewards; $\E\brackets{\norm{\widetilde{\nabla}_t}^2}$ is increasing whenever probability mass in $x_t$ is transferred from an arm $x_{t,i}$ to $x_{t,j} > x_{t,i}$, and thus we can obtain a bound in terms the expectation of a squared binomial random variable $X$ with $H = \frac{n}{\epsilon}$ trials, where each trial has value at most $1$ with probability $\frac{2\epsilon(n-1)}{n}$ (if any of the $n-1$ are sampled), and value $\frac{1}{H}$ otherwise. This yields:
\begin{align}
    \E \brackets{\norm{\widetilde{\nabla}_t}^2} \leq&\; H\parens{ \frac{2\epsilon(n-1)}{n} }\parens{1 - \frac{2\epsilon(n-1)}{n}} + \parens{\frac{2(H - 1)\epsilon(n-1)}{n} + 1 }^2 \notag \\   
    \leq&\; 2(n-1) +  (2n-1 )^2 \notag \\ 
    \leq&\; 4n^2. \label{eq:grad-sq-norm}
\end{align}
Over all $T$,  for any fixed $x^* \in \K_T$ we have:
\begin{align}
    \sum_{t=1}^T r_t^{\top}x^* - \hat{r}_t^{\top}x^* \leq&\;  \frac{H}{2} = \frac{n}{2 \epsilon}, \label{eq:seq-deviation}
\end{align}
as only fractional rewards from the last $H$ rounds are omitted from being counted appropriately in $ \sum_t \hat{r}_t$. We now analyze the regret of our algorithm with respect to the sequence $\{ \hat{r}_t \}$. For $x^* \in \K_T$ we have:
\begin{align}
  \sum_{t=1}^T \hat{r}_t^{\top} x^* - \sum_{t=1}^T \EE{\hat{r}_t^{\top}y_t}  \leq&\; \sum_{t=1}^T \hat{r}_t^{\top} x^*_{\epsilon} - \sum_{t=1}^T \EE{\hat{r}_t^{\top} {x}_t} + \sqrt{n}\epsilon  T  \tag{each $r_t$ is $\frac{\sqrt{n}}{2}$-Lipschitz} \\
  \leq&\; \sum_{t=1}^T {r^*_t}^{\top} x^*_{\epsilon}   - \sum_{t=1}^T {r^*_t}^{\top} {x}_t   + \sqrt{n} \epsilon T + \sum_{t=1}^T \sum_{i=1}^n \frac{\abs { \xi_{t,i} }}{x_{t,i}}  \tag{by \eqref{eq:seq-step-deviation}} \\
  \leq&\; \E\brackets{ \Reg_{COGD}(\widetilde{\nabla}_{1},\ldots, \widetilde{\nabla}_{T }) } + \sqrt{n} \epsilon T  + \sum_{t=1}^T \sum_{i=1}^n \frac{\abs { \xi_{t,i} }}{x_{t,i}}  \label{eq:cogd-martin} \\
  \leq&\; \frac{\eta}{2} \sum_{t=1}^T \E \brackets{ \norm{\widetilde{\nabla}_{t }}^2 }  + \frac{\sqrt{2}}{\eta} +  \sqrt{n} \epsilon T  + \sum_{t=1}^T \sum_{i=1}^n \frac{\abs { \xi_{t,i} }}{x_{t,i}} \tag{by Lemma \ref{lemma:ogd-contracting}}.
\end{align}
Line \eqref{eq:cogd-martin} holds by observing that our algorithm is equivalent to running Contracting Online Gradient Descent to the sequence $\{\widetilde{\nabla}_t\}$, where $\E \brackets{ \sum_{t=1}\widetilde{\nabla}_t } =   \sum_{t=1}r^*_t$ by \eqref{eq:exp-seq-match}.
As such, by \eqref{eq:grad-sq-norm} and \eqref{eq:seq-deviation} we have that
\begin{align*}
      \sum_{t=1}^T {r}_t^{\top} x^* - \sum_{t=1}^T \EE{{r}_t^{\top}y_t}  \leq&\;  2 \eta n^2 T  + \frac{\sqrt{2}}{\eta} +  3\sqrt{n} \epsilon T + \frac{n}{2\epsilon} + \sum_{t=1}^T \sum_{i=1}^n \frac{\abs { \xi_{t,i} }}{x_{t,i}}.  
\end{align*}

\end{proof}

\section{Omitted Proofs for Section 3}

\subsection{Proof of Lemma \ref{lemma:ird-menus}}
\label{sec:menutime-proof}

\begin{proof}
Let the menu time $\mu_i$ for each item be given by
    \begin{align*}
    \mu_i :=&\;  \frac{k \cdot \frac{x_i}{f_i(v)} }{  \sum_{j=1}^n  \frac{x_j}{f_j(v)} }.
\end{align*}
It is straightforward to see that $\sum_i \mu_i = k$. Intuitively, menu time corresponds roughly to the relative frequency with which an item must be included in the menu, where an item with $\mu_i=1$ must always be included in the menu; the amount of menu time ``charged'' for a menu is inversely proportional to the sum of item scores in the menu, which enables an ``apples to apples'' comparison between resulting selection probabilities. 

We first show that any $x \in \IRD(v, M)$ results in $\mu_i$ at most 1 per item.
    For any $x \in \IRD(v, M)$, consider an arbitrary convex combination of the menu-conditional item distributions given by items' scores $f_i(v)$, with the probability of each menu given by $p_K$. 
Allocate ``menu time units'' $\mu_K$ to each menu $K$ in proportion with $p_K / \sum_{i\in K} f_i(v)$, such that $\sum_K \mu_K = k$, and further let $\mu_{K,i} = \mu_K / k$ for each $i \in K$.  
    Observe that selection probability of an item $i$ is given by
    \begin{align*}
        x_i =&\; \sum_{K : i \in K} p_K \cdot \frac{f_i(v)}{\sum_{j \in K} f_j(v)} \\
        =&\; \frac{1}{Z} \sum_{K : i \in K} \frac{\mu_K}{k} \cdot {f_i(v)}\\
        =&\; \frac{f_i(v)}{Z} \sum_{K : i \in K} {\mu_{K,i}} \\
    \end{align*}
where $Z$ is a normalizing constant such that $\sum_{K} \mu_K = k$, and so we have that $\sum_{K} \mu_{K,i} \leq 1$ as each $\mu_K$ is positive. 
Further, we have that
\begin{align*}
    \sum_{K : i \in K} {\mu_{K,i}} =&\; Z \cdot \frac{x_i}{f_i(v)} \\
    =&\; \frac{x_i}{f_i(v)} \cdot \frac{k}{\sum_{j=1}^n \frac{x_j}{f_j(v)} }
\end{align*}
upon solving for $Z$ such that $\sum_{K} \mu_K = k$, which gives us that
\begin{align*}
    \sum_{K : i \in K} {\mu_{K,i}} =&\; \mu_i
\end{align*}
and yields  $\mu_i \leq 1$ for each item.

Next, we describe an algorithm \textup{\texttt{MenuDist}$(v, x, M)$} for constructing a menu distribution $z$ which yields $\E_{K\sim z}[p(K,v)] = x$ for any $x$ and $v$ satisfying $\max_i \mu_i \leq 1$, constructively showing that $x \in \IRD(v, M)$.
Given $x$, $v$, and $M$ which satisfy $\max_i \mu_i \leq 1$, we construct a menu distribution iteratively over $H \leq n$ {\it stages}. Let $\mu_i^1 = \mu_i$ be the initial remaining menu time for each item. At each stage $h \geq 1$ we will decrement the remaining time to $\mu_i^{h+1} \leq \mu_i^h$ for each item and track its change $\Delta^h_i = \mu^h_i - \mu^{h+1}_i$, with $\Delta^h = \sum_{i} \Delta_i^h$.  Further, with $\mu_j^h$ as the $k$th highest remaining value, let $k^h_{+} < k$ be the number of items $i$ with $\mu_i^h > \mu_j^h$, and $k^h_{*} \geq 1$ be the number of items $i$ with $\mu_i^h = \mu_j^h$.
At each stage, we maintain the following invariants:
\begin{enumerate}
    \item $k^{h+1}_{*} \geq k^h_{*} + 1$ if $k^h_{*} < n$;
    \item $\Delta_i^h = \frac{\Delta^h}{k}$ if $\mu_i^h > \mu_j^h$ (where $\mu_j^h$ is the $k$th highest value);
    \item $\Delta_i^h = \frac{(k - k^h_{+}) \Delta^h}{k^h_{*} \cdot k}$ if $\mu_i^h = \mu_j^h$;
    \item $\Delta_i^h = 0$ if $\mu_i^h < \mu_j^h$; and
    \item $\Delta_i^h = \mu_i^h$ if $k^h_{*} = n$.
\end{enumerate}
Intuitively, we are decreasing $\mu_i$ of each item with strictly larger $\mu_i^h$ than the $k$th highest at identical rates, corresponding to inclusion in the menu at that stage with probability $1$, and breaking ties among the $k$th highest with fractional inclusion in the remaining spots. We decrease until the number of such tied items increases by at least 1 (or deplete all remaining menu time if all items are tied), and then continue to the next stage. Observe that when $\max_i \mu_i \leq 1$, this results in $\mu^{H+1}_i = 0$ for all items; if an item has $\mu_i^h > \mu_j^h$  (only occurring when $k^h_{+} > 0$), its remaining menu time decreases at a $1/k$ fraction of the total rate of decrease, and no $\mu_i^h$ will ever drop below 0 as we break ties among the highest remaining, and the total amount of depleted menu time is at most $k$. Once $k^h_{+} = 0$ we only ever decrease $\mu_i^h$ for items tied for the highest remaining value, and thus successfully satisfy $\mu_i^{H+1} = 0$ while maintaining our invariants.

We now show how to construct a menu distribution $z$ from the quantities $\Delta_i^h$ which yields $p(z,v) = \E_{K\sim z}[p(K,v)] = x$. 
For a stage $h$, we construct a (multi)set of menus $S^h = \{K^h_s : s \in [\abs{S^h}] \}$ with size $\abs{S_h} = k^h_*$, where we can assume each $K^h_s \in S^h$ is distinct without loss of generality (e.g.\ by marking duplicates with some auxiliary notation).
Let $J^h = \{ i : \mu_i^h = \mu_j^h \}$ be the set of items tied for $k$th highest remaining; construct $S^h$ iteratively over $k^h_*$ steps by adding a menu $K_s^h$ which includes all items with $\mu_i^h > \mu_j^h$, and $k - k_{+}^h$ items from $J^h$ which belong to the fewest menus in $S^h$ thus far, breaking ties arbitrarily. There are a total of $(k - k_{+}^h)k^h_*$ inclusions of some item in $J^h$ to some menu in $S^h$, which results in each item $i \in  J^h$ being included in exactly $k - k_{+}^h$ menus in $S^h$ (as $(k - k_{+}^h)k^h_*$ is divisible by $k - k_{+}^h$), and the uniform distribution over $S^h$ then satisfies
\begin{align*}
    \Pr_{K^h_s \sim \Unif(S^h)} \brackets{i \in K^h_s} =&\; \begin{cases}
        1 & \mu_i^h > \mu_j^h \\ 
        \frac{k - k_{+}^h}{k^h_*} & \mu_i^h = \mu_j^h \\
        0 & \mu_i^h < \mu_j^h
    \end{cases}.
\end{align*}
Let the menu distribution $z^h$ be the distribution over $S^h$ given by 
\begin{align*}
    \Pr_{K \sim z_h }\brackets{ K = K^h_s } =&\; \frac{\sum_{i \in K} f_i(v)}{\sum_{s \in \abs{S^h}} \sum_{j \in K_s^h} f_j(v)},
\end{align*}
which yields 
\begin{align*}
    \Pr_{K \sim z^h}\brackets{ \text{Agent chooses i} } =&\; \sum_{K \in S^h : i \in K} \frac{f_i(v)}{\sum_{q \in K} f_{q}(v)} \cdot \frac{\sum_{q \in K} f_q(v)}{\sum_{s \in \abs{S^h}} \sum_{j \in K_s^h} f_j(v)} \\ 
    =&\; \frac{k \Delta_i^h}{\Delta^h} \cdot  \frac{f_i(v)}{\sum_{s \in \abs{S^h}} \sum_{j \in K_s^h} f_j(v)}  \\
    \overset{\Delta}{=}&\; \frac{k \Delta_i^h}{\Delta^h} \cdot  \frac{f_i(v)}{Z^h} 
\end{align*}
as the probability of an agent choosing $i$ conditional on being shown  a menu sampled from $z^h$.
Defining  $Z = \sum_{h=1}^H \frac{Z^h \Delta^h}{k}$,
let $z$ be the mixture over distributions $z^h$ with mass $\frac{Z^h \Delta^h}{Zk}$ on each. Sampling a menu $K \sim z$ yields
\begin{align*}
    \Pr_{K \sim z}\brackets{\text{Agent chooses i} } =&\; \sum_{h=1}^H \frac{Z^h \Delta^h}{Z k} \cdot \Pr_{K \sim z^h}\brackets{ \text{Agent chooses i} } \\ 
    =&\;\sum_{h=1}^H \frac{Z^h \Delta^h}{Z k} \cdot \frac{k \Delta_i^h}{\Delta^h} \cdot  \frac{f_i(v)}{Z^h}  \\ 
    =&\; f_i(v) \cdot \frac{\mu_i}{Z}  \tag{$\sum_{h} \Delta_i^h = \mu_i$} \\
    =&\; x_i \tag{$\mu_i \propto x_i / f_i(v)$},
\end{align*}
where  
$Z = \sum_{j=1}^n k \frac{x_j}{f_j(v)}$ then holds by proportionality as $x$ and $p(z, v)$ are both probability distributions over $[n]$. By construction, it follows that $x \in \IRD(v, M)$.

The algorithm \texttt{MenuDist}$(v, x, M)$ which implements this construction can be run in $\poly(n)$ time, as the quantities $\Delta_i^h$ are computed over $H \leq n$ rounds each requiring $O(n)$ computation (after an initial sort of descending $\mu_i$ values), and each set $S^h$ contains $k^h_* \leq n$ menus, each of which can be constructed in $O(k)$ time (by adding elements from $J^h$ in a round-robin fashion) while maintaining the quantities necessary to compute the final normalizing constants. 
Sampling can be implemented efficiently as well, e.g.\ by sampling from a uniform distribution and thresholding appropriately.
\end{proof}

\subsection{Effective Memory Horizons}
Here we give results for bounding the change in memory as rounds progress.

\begin{lemma}[Bounding Memory Drift]\label{lemma:drift-bound-1}
For any $\gamma \in (0,1)$, $g \in (0,\gamma]$,  $t \geq 1$, and $w \geq 1$ such that $g^w \geq 1 - 2\beta$ and $g^{t+w-1} \leq \frac{1}{2}$, we have that
$d_{TV}(v_t, v_{t+w})$ is at most $\beta$.
\end{lemma}
\begin{proof}
We can express the memory vector $v_{t+w}$ as
\begin{align*}
    v_{t+w} =&\; \frac{ \sum_{s=t}^{t+w-1} \gamma^{t+w-s-1} \cdot i_s + \sum_{s=1}^{t-1} \gamma^{t+w-s-1} \cdot i_s }{\sum_{s=1}^{t+w-1} \gamma^{t+w-s-1} } \\
    =&\; \frac{\sum_{s=t}^{t+w-1} \gamma^{t+w-s-1} \cdot i_s}{\sum_{s=1}^{t+w-1} \gamma^{t+w-s-1} } + v_{t} \cdot \parens{1 - \frac{\sum_{s=t}^{t+w-1} \gamma^{t+w-s-1}}{\sum_{s=1}^{t+w-1} \gamma^{t+w-s-1}} } 
\end{align*}
which then yields
\begin{align}
    d_{TV}(v_t, v_{t+w}) \leq&\; \frac{\sum_{s=0}^{w-1} \gamma^{s} }{\sum_{s=0}^{t+w-2} \gamma^{s} }  \nonumber \\
    =&\; \frac{1 - \gamma^w}{1 - \gamma^{t+w-1}} \nonumber \\
    \leq&\; \frac{1 - g^w}{1 - g^{t+w-1}} \label{eq:window-g-gamma-grad} \\ 
    \leq&\; \beta. \nonumber
\end{align}
Step \ref{eq:window-g-gamma-grad} follows from the fact that
\begin{align*}
    \frac{\partial}{\partial \gamma} \parens{ \frac{1 - \gamma^w}{1 - \gamma^{t+w-1}} } =&\; \frac{-w \gamma^{w-1}(1 - \gamma^{t+w-1})  + (1 - \gamma^w) (t+w-1)\gamma^{t+w-2}   }{(1 - \gamma^{t+w-1})^2} \\
    =&\; \frac{ \gamma^{w-1} \parens{  (t+w-1)\gamma^{t-1} - (t - 1)  \gamma^{t+w-1}  - w  } }{(1 - \gamma^{t+w-1})^2} \\
    \leq&\; 0,
\end{align*}
as the function  $(t+w-1)\gamma^{t-1} - (t - 1)  \gamma^{t+w-1}  - w$ is increasing over $\gamma \in [0,1]$ from $-w$ to $0$ (and thus the derivative for (\ref{eq:window-g-gamma-grad}) is negative at any $\gamma \in (0,1)$).
\end{proof}

We can use this to obtain an upper limit on $w$ in terms of $c$ such that this bound holds.
\begin{lemma}\label{lemma:drift-bound-2}
    For $\gamma \geq 1 - 1/T^c$, $t \geq T^c$, and $w \leq \beta \cdot T^c$, we have that $d_{TV}(v_t, v_{t+w}) \leq \beta$.
\end{lemma}
\begin{proof}
Let $g = 1 - \frac{1}{T^c}$, where we have that $g^{(T^c)} \leq \frac{1}{e} \leq \frac{1}{2}$. 
Further, we have:
\begin{align*}
    \log \parens{\frac{1}{1 - 2\beta}} \geq&\; 2\beta \\
    \geq&\;   \frac{2w}{T^c}  \\
    \geq&\; w\log \parens{\frac{1}{1 - \frac{1}{T^c}}} \tag{ for $\frac{1}{T^c} \leq \frac{1}{2}$} \\
    =&\; \log \parens{ \frac{1}{g^w} }.
\end{align*}
As $\log(1/x)$ is decreasing in $x$ we have that $g^w \geq 1 - 2\beta$, which yields the result via Lemma \ref{lemma:drift-bound-1}. 
\end{proof}

\subsection{Analysis for Algorithm \ref{alg:long-eird}: Targeting $\EIRD$}
\renewcommand{\thealgorithm}{2}
\begin{algorithm}
    \caption{(Targeting $\EIRD$ for Smooth Models).}
    \begin{algorithmic}
        \STATE Let $c^* = \min(c, 3/4)$, $\epsilon = \tilde{O}(nL^{1/4} \lambda^{-3/2} T^{-c/4})$, $Q = \tilde{O}(\frac{n^2}{\lambda^4\epsilon^2})$, and $\eta = (nT)^{-1/2}$.
        \STATE Initialize $q = 0$, $v_0 = v^* = \mathbf{u}_n$,  $F_{i} = \frac{C}{n}$ for $i\in [n]$, $M^* = \{F_i\}$
        \STATE Initialize $\dbg$ for $\epsilon, \eta$.
        \WHILE{$t\leq T$}
            \IF{$t < T^{c^*}$}
            \STATE Show arbitrary menu $K_t$ to agent
            \ELSIF{$t \geq T^{c^*}$ and either $\norm{ v_t - v^* }_1 \geq \frac{\lambda\epsilon}{2nL}$ or $q = 0$}
            \FOR{$b \in \{0,\ldots ,\lceil \frac{n-1}{k-1} \rceil\}$}
            \STATE Show agent menu $K_b = \{1\} \cup \{b(k-1)+2,\ldots,(b+1)(k+1) + 1\}$ for $Q$ rounds
            \STATE Let $\hat{F}_i = \parens{ \text{\# times } i_t = i} / \parens{ \text{\# times }i_t = 1}$ within the $Q$ rounds, for $i \in K_b$
            \ENDFOR
            \STATE Set $F_i = \frac{C \cdot \hat{F}_i}{\sum_{j=1}^n \hat{F}_j}$ for each $i\in[n]$, $M^* = \{F_i\}$
            \STATE Set $v^* = v^t$, increment $q$ by $\lceil \frac{n-1}{k-1} \rceil \cdot Q$, set $\K_{t-q+1} = \K_{t - q} \cap \IRD(v^*, M^*)$ for $\dbg$
            \ELSE
            \STATE Get $x_{t-q}$ from $\dbg$
            \STATE Let $z_t = \texttt{MenuDist}(v_t, x_{t-q}, M^*)$, sample menu $K_t \sim z_t$
            \STATE Show $K_t$ to agent, update $\dbg$ with observed $i_t$ and $r_t(i_t)$
            \ENDIF 
            \STATE Set $v_{t+1} = U(v_t, i_t, t)$, for each round counted by $q$ if necessary
        \ENDWHILE
    \end{algorithmic}
\end{algorithm}

\begin{customthm}{5}
For an agent with a $(\lambda, L)$-smooth preference model $M$ for $\lambda \geq k/n$, and $\gamma$-discounted memory for $\gamma \geq 1 - \frac{1}{T^c}$ and $c \in (0,1]$, Algorithm \ref{alg:long-eird} obtains regret bounded by
\begin{align*}
\max_{x^* \in \EIRD(M)} \sum_{t=1}^T r^{\top}_t x^* - \E\brackets{\sum_{t=1}^T r_t(i_t) }
    =&\; \tilde{O}\parens{(n/\lambda)^{3/2} L^{1/4} \cdot T^{1 - c/4} }.
\end{align*}
\end{customthm}

\begin{proof}
In each round, outside of those spent updating our estimates of preference (counted by $q$), we receive a target distribution $x_t$ from $\dbg$, and we construct a menu distribution $z_t$ which aims to induce a choice distribution $x_t$ on behalf of the agent.
Recall that we assume $\gamma$ is known, and so we can exactly track the agent's memory vector $v_t$ across rounds; we note that this result can be extended to the case where only a lower bound on $\gamma$ is known by checking the condition on $\norm{v_t - v^*}_1$ across all possible values of $\gamma$ in each round.
Observe that if our preference estimates $F_i$ were to always exactly track the agent's true preferences $f_i(v_t)$, and yield exact representations of $\IRD(v_t)$ in each round, then we would have perturbations $\xi_t = \mathbf{0}$ to each target distribution $x_t$ by the guarantee of Lemma \ref{lemma:ird-menus}, and a decision set $\K_{T}$ which contains $\EIRD$. As such, our regret would be immediately bounded by
\begin{align*}
   \max_{x^* \in \EIRD(M)} \sum_{t=1}^T r^{\top}_t x^* - \E\brackets{\sum_{t=1}^T r_t(i_t) } \leq&\; \Reg_{\EIRD}(\dbg; T, \epsilon, \eta)  + T^{1 - c/4} + q = \tilde{O}(T^{1 - c/4} + q) 
\end{align*}
as $T^{c^*} \leq T^{1 - c/4}$ for any $c\in (0,1]$ (ignoring non-$T$ terms).
Here, note that $\EIRD$ is non-empty and contains $\mathbf{u}_n$ by Lemma [1], as $\lambda \geq \frac{k}{n}$.
The remainder of our analysis will focus on showing that the perturbations $\xi_t$ from preference estimate imprecision remain small, and that the estimation time $q$ does not grow too large. 

\paragraph{Case 1: $c^* = c$.}
For any $c \leq 3/4$, the agent's memory is ``saturated'' by the time we conclude showing arbitrary menus and proceed to our alternation between learning and optimization.
Let $\epsilon$ satisfy
\begin{align*}
    \epsilon \geq&\; 8nL^{1/4} \lambda^{-3/2}\log(2T/\delta)^{1/4}T^{-c/4} 
\end{align*}
with $\delta = 1/T$. Observe that the following hold: 
\begin{align*}
    \frac{\lambda^2 \epsilon}{8nkL} \cdot T^c \geq&\; \frac{64n^2(n-1) \log(2T /\delta )}{(k-1)\lambda^4 \epsilon^2}
     \tag{$\epsilon^3 \geq \epsilon^4 \geq 512 n^4 L \lambda^{-6} \log(2T /\delta)  T^{-c}$} \\
    \frac{\lambda\epsilon^2 }{4nL} \cdot T^c \geq&\; \frac{64n^2 (n-1) \log(2T /\delta )}{(k-1)\lambda^4 \epsilon^2}   \tag{$\epsilon^4 \geq 256 n^4k^{-1} L\lambda^{-5} \log(2T /\delta ) \cdot T^{-c}$} 
\end{align*}
By Lemma \ref{lemma:drift-bound-2}, in each of the $Q \cdot \lceil \frac{n-1}{k-1} \rceil$ prior to updating score estimates $F_i$, we have
\begin{align*}
    \abs{ \frac{f_i(v^*)}{\sum_{j\in K_b} f_i(v^*)} - \frac{ f_i(v_t)}{\sum_{j\in K_b} f_i(v_t)} } \leq&\; \frac{\lambda^2 \epsilon}{8n},
\end{align*}
by $(\lambda, L)$-smoothness for $\{f_i\}$, as $Q \cdot \lceil \frac{n-1}{k-1} \rceil \leq \frac{\lambda^2 \epsilon}{8nk L} \cdot T^c$.
By a Hoeffding bound, we then have 
\begin{align*}
    \Pr \brackets{ \abs{  {(\text{\# times $i \in K_b$ chosen })}  - \frac{Q f_i(v^*)}{\sum_{j\in K_b} f_i(v^*)} } > \frac{Q \lambda^2 \epsilon}{4n}  } \leq&\; 2 \exp\parens{ -\frac{Q^2 \lambda^4 \epsilon^2 }{64 n^2 Q} } \\
    \leq&\; \frac{\delta}{T},
\end{align*}
as $Q \geq \frac{64n^2 \log(2T /\delta )}{\lambda^4 \epsilon^2}$. If this holds for all $i$ across each $K_b$ (including each instance for $i=1$), we then have 
\begin{align*}
    \abs{\hat{F}_i - \frac{f_i(v^*)}{f_1(v^*)} } \leq&\; \frac{\lambda \epsilon}{2n}, 
\end{align*}
as $f_1(v^*) \geq \lambda$.
Observe that the normalization to $\{F_i\}$ is equivalent to rescaling the empirical frequencies observed for each $K_b$ such that their scores agree on $i=1$ and the largest score across all items is at most 1, which will not increase relative error for any item. As such, we have
\begin{align*}
    \abs{F_i - f_i(v^*)} \leq&\; \frac{\lambda \epsilon}{2n}
\end{align*}
as well. Further, in any subsequent round where $\norm{v_t - v^* }_1 \leq \frac{\lambda \epsilon}{2nL}$ we have
\begin{align*}
    \abs{F_i - f_i(v_t)} \leq&\; \abs{F_i - f_i(v^*)} + \abs{f_i(v_t) - f_i(v^*)} \\
    \leq&\; \frac{\lambda \epsilon}{n},
\end{align*}
as each $f_i$ is $L$-Lipschitz with respect to the $\ell_1$ norm, and so for all rounds $t$ where $\dbg$ is played, we have $\abs{ F_i - f_i(v_t)} \leq \frac{\lambda \epsilon}{n}$ for each $i$ with probability $1 - \delta$, which contributes at most 1 to our expected regret if $\delta = 1/T$.
By Lemma \ref{lemma:drift-bound-2}, this holds for at least $\frac{\lambda \epsilon}{4nL} \cdot T^c$ subsequent rounds. We now show that this yields a bound of $\abs{ \xi_i } \leq \frac{\epsilon}{n} x_i$ in each round for $\dbg$.
\begin{lemma}\label{lemma:eird-dist-bound}
Suppose that $\abs{ F_i - f_i(v_t)} \leq \frac{\lambda \epsilon}{n}$ for each $i$. Then, for any menu distribution $z$ which realizes a choice distribution $x$ for scores $\{F_i\}$, the choice distribution $y$ for scores $\{f_{i}(v_t)\}$ satisfies $\abs{x_i - y_i} \leq \frac{\epsilon}{n} x_i$.
\end{lemma}
\begin{proof}
Let $y_i = \frac{f_i(v_t)}{F_i} x_i$ for each $i$. Observe that $(x, \{F_i\})$ and $(y, \{f_{i}(v_t)\})$ yield the same menu time values:
\begin{align*}
    \frac{k \cdot \frac{x_i}{F_i}}{\sum_{j=1}^n \frac{x_j}{F_j}} =&\; \frac{k \cdot \frac{y_i}{f_{i}(v_t)}}{\sum_{j=1}^n \frac{y_j}{f_{j}(v_t)}}.
\end{align*}
By the first construction in Lemma \ref{lemma:ird-menus}, this implies that any menu distribution $z$ realizing $x$ under $\{F_i\}$ satisfies: 
\begin{align*}
    \frac{x_i}{F_i} =&\; \sum_{K \in z : i \in K}  \frac{z_{K}}{\sum_{j \in K} F_{j}} 
\end{align*}
and so the same distribution will satisfy  
\begin{align*}
    \sum_{K \in z : i \in K}  \frac{z_{K}}{\sum_{j \in K} f_{i}(v_t)} =&\; \frac{\frac{f_i(v_t)}{F_i} x_i }{f_i(v_t)} \\
    =&\; \frac{y_i }{f_i(v_t)}
\end{align*}
under $\{f_{i}(v_t)\}$, yielding a choice distribution of $y$. As such, we have that 
\begin{align*}
    \abs{ x_i - y_i} \leq&\;  \abs{ \frac{F_i - f_i(v_t)}{F_i}}x_i \\
    \leq&\; \frac{\epsilon}{n} x_i.
\end{align*}    
\end{proof} Given a set of scores $M^* = \{F_i\}$, the set of feasible distributions can be expressed via linear constraints as
\begin{align*}
    \IRD(M^*) =&\; \{x \in \Delta(n) : \frac{k x_i}{F_i} \leq \sum_{j=1}^n  \frac{k x_j}{F_j} \}.
\end{align*}
To ensure that we never remove points which belong to $\EIRD$ from our target set due to imprecision in $\IRD$ estimates, we can relax each target set by  $\frac{\epsilon}{n}$ along each dimension:
\begin{align*}
        \IRD_{\eps}(M^*) =&\; \braces{  \brackets{\parens{1-\frac{\epsilon}{n}}x , \parens{1-\frac{\epsilon}{n}}x} \cap \Delta(n) : x \in \IRD(M^*) }. 
\end{align*}
However, such points will not be chosen by by our algorithm anyway, due to the $\epsilon$-contraction from $\K_t$ to $\K_{t, \epsilon}$ in $\dbg$.
By the bound on $\epsilon$, the total time spent updating our estimates $F_i$ (counted by $q$) is at most $\epsilon T$, as at least
\begin{align*}
    \frac{\lambda\epsilon }{4nL} T^c \geq&\; \frac{64n^2(n-1) \log(2T /\delta )}{(k-1)\lambda^4 \epsilon^{3}} 
\end{align*}
rounds elapse between each learning stage of length $Q \cdot \lceil \frac{n-1}{k-1} \rceil = \frac{64n^2(n-1) \log(2T /\delta )}{(k-1)\lambda^4 \epsilon^{2}}$.
As such, our total expected regret can be bounded by
\begin{align*}
\max_{x^* \in \EIRD(M)} \sum_{t=1}^T r^{\top}_t x^* - \E\brackets{\sum_{t=1}^T r_t(i_t) } \leq&\; \tilde{O}\parens{ (n/\lambda)^{ 3/2 }L^{1/4} T^{-c/4} }
\end{align*}
assuming worst-case reward during each of the first $T^c$ stages as well as the $q \leq \epsilon T$ stages spent estimating $F_{i}$.

\paragraph{Case 2: $c > c^*$.} When $c > c^*$ our analysis proceeds similarly, with the exception we can no longer uniformly bound the number of rounds in which $\norm{v_t - v_1} \leq \frac{\lambda\epsilon}{2nL}$ between updates to $F_i$. However, we show that the amortized learning time is similar
and that $q$ can still be bounded by $\tilde{O}(T^{1 - c/4})$. Let $c^*(t) = \min\parens{ c, {\log(t)}/{\log(T)} }$ be the parameterization of $c^*$ satisfying $t = T^{c^*(t)}$ for $c^*(t) \leq c$; note that we may still apply Lemma \ref{lemma:drift-bound-2} according to $c^*(t)$ rather than $c$. At any $t \geq T^{3/4}$, $c^*(t) \geq 3/4$ suffices to bound the requisite change in $v_t$ which occurs during the $Q\cdot \lceil \frac{n-1}{k-1} \rceil = $ learning rounds as we did in Case 1,  as
\begin{align*}
    \frac{\lambda^2 \epsilon}{8nk L} \cdot T^{c^*(t)} \geq&\; \frac{64n^2(n-1) \log(2T /\delta )}{(k-1)\lambda^4 \epsilon^2}.
     \tag{$\epsilon^3 \geq 512 n^4 L \lambda^{-6} \log(2T /\delta)  T^{-3/4}$} \\
\end{align*}
For any range $t = [2^h,2^{h+1}]$ for $h \geq \log(T^{3/4})$ and $h \leq \log(T^{c})$, our bound on $v_t$ holds for at least 
\begin{align*}
    \frac{\lambda\epsilon}{4nL} \cdot T^{c^*(t)} \geq \frac{\lambda\epsilon}{4nL} \cdot 2^{h}
\end{align*}
rounds, and so at most $\frac{4nL}{\lambda\epsilon}$ learning stages occur during this window. For $h \geq \log(T^{c})$ we can apply the same bound as in Case 1. Across all $h  \leq \log(T)$ we have that
\begin{align*}
    q \leq&\; \frac{4nL}{\lambda\epsilon} \cdot \frac{64n^2(n-1) \log(2T /\delta )}{(k-1)\lambda^4 \epsilon^2} \log(T) \\
    \leq&\; \epsilon T\log(T), 
\end{align*}
yielding a total regret bound matching that in Case 1 up to a $\log(T)$ factor.
\end{proof}

\section{Omitted Proofs from Section \ref{sec:biggamma-pi}}
\label{sec:proofs-ss}
\subsection{Proof of Lemma \ref{lemma:pseudoinc-ird}}

\begin{proof}
    From Lemma \ref{lemma:ird-menus}, it suffices to show that the menu time $\mu_i$ for any such point $x$ is at most $1$. Given that $v$ and $x$ lie inside $\Delta^{\phi}(n)$ and that each $f_i$ is pseudo-increasing, we have that
\begin{align}
    \frac{x_i}{ ((1 - \lambda)v_i + \lambda) \sigma } \leq  \frac{x_i}{f_i(v)} \leq \frac{\sigma x_i }{ (1 - \lambda)v_i + \lambda } \label{eq:pseudoinc-bound}
\end{align}
and that $x_i \leq \min( v_i + \lambda \phi / \sqrt{2}, 1 - \phi  (n-1)/n)$ as $x \in B_{\lambda \phi}(v)$. 
Recall that the menu time $\mu_i$ for $x$ is given by
\begin{align*}
    \mu_i(x) =&\; \frac{k \cdot \frac{x_i}{f_i(v)}}{\sum_{j=1}^n \frac{x_j}{f_j(v)}}.
\end{align*}
First we show that this numerator is at most $k\sigma$. We have that
\begin{align*}
    x_i \leq&\; v_i + \lambda \phi / \sqrt{2} \\
    \leq&\; (1-\lambda)v_i + \lambda(1 - \phi(n-1)/n) + \lambda \phi / \sqrt{2} \\
    \leq&\; (1-\lambda)v_i + \lambda \tag{for $n \geq 4$}
\end{align*}
which yields that $k x_i / f_i(v) \leq k\sigma$ by \eqref{eq:pseudoinc-bound}.
We can also lower-bound the menu denominator by $k\sigma$. 
Let $\alpha_j = v_j - x_j$, where we have:
\begin{align*}
   \sum_{j=1}^n \frac{x_j}{f_j(v)} \geq&\; \sum_{j=1}^n \frac{v_j - \alpha_j }{((1 - \lambda)v_j + \lambda) \sigma}.
\end{align*}
Differentiating with respect to $v_j$ for any term, we have:
\begin{align*}
    \frac{\partial}{\partial v_j} \frac{v_j - \alpha_j}{((1 - \lambda)v_i + \lambda) \sigma} =&\; \frac{((1 - \lambda)v_i + \lambda) \sigma - (v_j - \alpha_j)(1-\lambda)\sigma}{\sigma^2 ((1-\lambda)v_i + \lambda)^2 } \\
    =&\; \frac{\lambda  + \alpha_j (1-\lambda)}{\sigma ((1-\lambda)v_i + \lambda)^2 } 
\end{align*}
which is positive for any $v_j$, and so each numerator term is increasing for any fixed $\alpha_j$ (as $\abs{\alpha_j} \leq \lambda$). As such, each term is minimized over valid $v_j$ and $x_j$ when $x_j = \phi/n$ and $v_j = \phi/n + \lambda \phi /\sqrt{2}$, yielding 
\begin{align*}
   \sum_{j=1}^n \frac{x_j}{f_j(v)} \geq&\;  \frac{\phi }{((1 - \lambda)(\phi/n + \lambda \phi /\sqrt{2}) + \lambda) \sigma} \\
   \geq&\; \frac{ (k\sigma) \cdot \phi }{k\sigma^2 \lambda + k\sigma^2 (1 - \lambda) \phi / n  + k \sigma^2 (1 - \lambda) \phi \lambda / \sqrt{2}} \\
   \geq&\;  \frac{ (k\sigma) \cdot \phi }{\phi/4 +  \phi /2  + \phi^2 / (4\sqrt{2})} \tag{$\phi \geq 4k\lambda\sigma^2$, $\sigma^2 \leq n/(2k)$} \\
   \geq&\; k\sigma. \tag{$\phi \leq 1$}
\end{align*}

As such, we have that
\begin{align*}
     \frac{k \cdot \frac{x_i}{f_i(v)}}{\sum_{j=1}^n \frac{x_j}{f_j(v)}} \leq&\; 1,
\end{align*}
and thus $x \in \IRD(v, M)$.
\end{proof}

\subsection{Analysis for Algorithm \ref{alg:long-ss}: Targeting $\Delta^{\Phi}(n)$ for Scale-Bounded Models}

\begin{customthm}{9}[Scale-Bounded Discounted Regret Bound]
For any agent with a preference model $M$ which is $(\sigma, \lambda)$-scale-bounded and $(\frac{ \lambda}{ \sigma}, L)$-smooth with $\sigma \leq \sqrt{n/(2k)}$, and with $\gamma$-discounted memory for $\gamma = 1 - \frac{1}{T^c}$ for $c \in (0,1)$, Algorithm \ref{alg:long-ss} obtains 
regret
\begin{align*}
    \max_{x^* \in \Delta^{\phi}(n)} \sum_{t=1}^T r^{\top}_t x^* - \E\brackets{\sum_{t=1}^T r_t(i_t) }  =&\; \tilde{O}\parens{(n^{4}L \parens{ T^{1 - c/2} + T^{1/2 + c/2}  }}
\end{align*}
with respect to the $\phi$-smoothed simplex, for $\phi = {4k\lambda \sigma^2}$.
\end{customthm}

\renewcommand{\thealgorithm}{2}
\begin{algorithm}
    \caption{(Targeting $\Delta^{\phi}$ for Scale-Bounded Models).}
    \begin{algorithmic}
        \STATE Let $\epsilon = \tilde{O}(n^4L \cdot \max( T^{-c/2}, T^{c/2 - 1/2} ))$, $S = \tilde{O}(n^{3/2} T^c )$,  $\eta = \tilde{O}(n^{-1/2}L \cdot T^{c/2 - 1/2})$.
        \STATE \texttt{--- burn-in ---}
        \FOR{$t = 1$ to $T^c$}
            \STATE Show agent menu $K = \{1,\ldots,k\}$
        \ENDFOR
        \WHILE{$\max_i \abs{v_{t,i} - \frac{1}{n}} \geq \frac{\epsilon}{4n^2L\sigma}$ }
            \STATE Show agent $k$ items with smallest $v_{t,i}$, choosing randomly among ties up to $T^{-c}$
        \ENDWHILE
        \STATE \texttt{--- initial learning ---}
        \FOR{$T^c$ rounds}
            \STATE Let $F_{t,i} = \sigma^{-1}((1 - \lambda)v_{t,i} + \lambda)$ if $v_{t,i} < \frac{1}{n}$, else $F_{t,i} = \sigma ((1 - \lambda)v_{t,i} + \lambda)$
            \STATE Let $z_t = \texttt{MenuDist}(v_t, \mathbf{u}_n,  \{F_{t,i}\})$, show agent $K_t \sim z_t$
        \ENDFOR
        \STATE Set $F_{i}=\sum_{t } \mathbf{1}(i_t=i) \cdot F_{t,i} / (\frac{1}{n} T^c)) \cdot C(\sum_{j=1}^n F_{t,j})^{-1}$, set $v^* = v_t$
        \STATE \texttt{--- optimization ---}
        \STATE Initialize $\ogd$ over $\Delta^{\phi}_{\epsilon}(n)$ for $T/S$ rounds with $\eta$, set  $x_1 = v^* := v_t$.
        \FOR{$s = 1$ to $T/S$}
            \STATE Receive $x_s$ from $\ogd$
            \FOR{$S$ rounds do}
            \IF{$\norm{v_t - v^*}_1 \geq \epsilon / (2nL)$}
            \STATE Let $\tilde{v} = v_t$
            \STATE Show agent $K_t \sim z = \texttt{MenuDist}(v_{t}, \tilde{v}, \{F_i\}) $ for $T^c / L^2$ rounds
            \STATE Set $F_i = \sum_{t} \mathbf{1}(i_t=i) \cdot F_i / ( \tilde{v}_i T^c /L^2 ) \cdot C(\sum_{j=1}^n F_{j})^{-1}$, set $v^* = v_t$
            \ENDIF
            \STATE Show agent $K_t \sim z_t = \texttt{MenuDist}(v^*, x_s, \{F_i\}) $
            \ENDFOR
            \STATE Set $\tilde{\nabla}_s = \sum_{h = t-S + 1}^t  e_{i_h} r_h(i_h) /( x_{h,i} S ) $, update $\ogd$
        \ENDFOR
    \end{algorithmic}
\end{algorithm}

\paragraph{Burn-in.}
Here we analyze the behavior of the algorithm in the initial $\tilde{O}(T^c)$ rounds, and show that each $v_{t,i}$ approaches $\frac{1}{n}$ with high probability. At $t = T^{c}$, the memory vector $v_t$ is entirely concentrated on $k$ items, and at least $n - k$ others have $v_{t,i} = 0$. We show that by showing the $k$ items with the smallest $v_{t,i}$ during each of the next $\tilde{O}(T^{c})$ rounds, we reach a memory vector $v_t \in \brackets{ \frac{1}{n} \pm \frac{\epsilon}{4n^2L\sigma}}^n$ with high probability. Suppose $v_{t,i}$ at least $T^{-c}$ below that for all but at most $k-1$ items, and each of the original $n-k$ smallest items have values $v_{t,j}$ within $\lambda/2$. Then, the probability that $i$ is chosen in round $t$ is at least
\begin{align*}
    y_{t,i} \geq&\; \frac{\lambda}{k\sigma^2(\lambda + (1 - \lambda)(\lambda/2))} \\ 
    \geq&\; \frac{4}{3 n }.
\end{align*}
As such, $v_{t,i}$ approaches $\frac{1}{n}$ faster than the average $v_{t,j}$ until it is no longer among the smallest $k$.
For $\lambda > T^{-c/2}$ the probability of $v_{t,i}$ falling more than $\lambda/2$ (or any constant) below any of the other $n-k$ decays exponentially, and thus the expected value of each $v_{t,i}$ at $2T^c$ is $\frac{1}{n} + O(T^{-c})$ after enough rounds to decrease the memory weight of the first $T^c$ rounds to $O(T^{-c})$, which occurs after $O(T^{c}\log(T^c))$ rounds following the analysis of Lemma [...14]. By Azuma's inequality, considering the martingale tracking the deviation between $v_{t,i}$ and its conditional expectation, this holds with high probability as well, up to $\frac{1}{n} + \tilde{O}(T^{-c/2}) = \pm\frac{\epsilon}{4n^2L\sigma}$.
The event that this fails to hold within $\tilde{O}(T^c)$ rounds contributes $O(1)$ to the total expected regret. Throughout the proof, we hide log terms and some constants inside $\tilde{O}$ notation.

\paragraph{Initial learning.}
Here we leverage the scale-bounded preference structure to obtain efficient estimators for scores near the current memory vector. With $v_t \in \brackets{ \frac{1-\alpha}{n}, \frac{1+\alpha}{n}}^n$, for any $\lambda > 0$ and sufficiently small $\alpha$ we have
\begin{align*}
    \frac{f_i(v_t)}{f_j(v_t)} \leq&\; \frac{n (\lambda + (1 - \lambda) \frac{1+ \alpha}{n})}{2k (\lambda + (1 - \lambda) \frac{1- \alpha}{n})} \leq \frac{n}{3k}
\end{align*}
and so for any distribution $x \in \brackets{ \frac{1}{2n}, \frac{3}{2n}}^n$  uniform distribution we have menu times
\begin{align*}
    \mu_i \leq&\; \frac{k\cdot 3n/2}{n \cdot 3k/2} \leq 1
\end{align*} 
for each $i$ according to the true scores $f_i(v_t)$. Our aim here is to learn accurate estimates of each of these scores. Observe that each of our proposed set of scores $\{F_{t,i}\}$ satisfies the scale-bounded conditions, and contains $\mathbf{u}_n \in \IRD(v_t, \{F_{t,i}\})$; as before, 
we can again take $v_t \in \brackets{ \frac{1}{n} \pm \frac{\epsilon}{4n^2L\sigma}}^n$ to hold with high probability over each of the $T^c$ rounds, as every $v_{t,i}$ moves closer to $\frac{1}{n}$ in expectation every round (up to $T^{-c}$), and thus the martingale tracking maximum deviation of the memory vector from expectation in any round under this process is bounded by $\tilde{O}(T^{c/2})$. Given this, we can obtain an unbiased estimator of each $f_i(v_t)$; we initially assume that $\sum_{i}f_i(v_t) = \sum_{i}F_{t,i}$ for each trial, and will subsequently correct for this in our sample aggregation by renormalizing such that $\sum_{i}F_{t,i} = C$.
\begin{lemma}
    At any $t$, if sampling from a menu distribution which generates a choice 
    distribution $x_t$ at $v_t$ according to scores $\{F_{t,i}\}$, with $\abs{ x_{t,i} - v_{t,i} } \leq \lambda$, then an unbiased estimator of each true preference score $f_i(v_t)$ is given by
\begin{align*}
    \E\brackets{ \frac{F_{t,i}}{x_{t,i}} \cdot \mathbf{1}( i \textup{ is chosen}) } =&\; f_i(v_t),
\end{align*}
with range bounded by $\frac{n}{2k\sigma}$ if $x_t \in \Delta^{\phi}(n)$ and $F_{t,i}$ satisfies scale-bounded constraints for $x_t$ at $v_t$. 
\end{lemma}

\begin{proof}
Recalling Lemma \ref{lemma:eird-dist-bound}, we have that the expected choice distribution $y_t$ satisfies 
\begin{align*}
    y_{t,i} =&\; \frac{f_i(v_t) x_{t,i}}{F_{t,i}},
\end{align*}
and rearranging gives us the estimator, as both $x_{t,i}$ and $F_{t,i}$ here are fixed. We can bound the range using the properties of scale-bounded functions and conditions for $\phi$:  
\begin{align*}
   \frac{F_{t,i}}{x_{t,i}} \leq&\; \frac{\sigma(\lambda + (1 - \lambda)(x_i + \lambda) )}{x_i} \\
   \leq&\; \frac{\sigma(\lambda n + (1 - \lambda)4k\lambda \sigma^2)}{4k\lambda \sigma^2} \tag{$x_i \geq \phi/n \geq 4k\lambda\sigma^2/n$} \\
   \leq&\; \frac{n}{2k\sigma}.
\end{align*}
\end{proof}
By the scale-bounded condition, the quantities $\sum_{i}f_i(v_t)$ and $\sum_{i}F_{t,i}$ are within a factor of $\sigma^2$,
and so renormalizing to $C$ can only increase each estimator's range to $\frac{n\sigma}{2k}$.
Normalizing the assumed sum of each trial $\sum_{i}F_{t,i}$ to $C$, as we do when aggregating to estimate $F_i$, yields a sum of random variables, each of whose mean is $f_i(\mathbf{u}_n)$.  
Applying Azuma's inequality to the sequence of trials for $T^c \geq \tilde{O}(n^8 L^2 / \epsilon^2) \geq \tilde{O}(n^6\sigma^4/(k\epsilon)^2)$
suffices to yield $\abs{f_i(v^*) - F_i} \leq \frac{\epsilon}{4n^2\sigma}$; additionally, given the Lipschitz condition on $f_i$ and that $v^* \in \brackets{ \frac{1}{n} \pm \frac{\epsilon}{4n^2L\sigma}}^n$, we also have that
$\abs{f_i(\mathbf{u}_n) - F_i} \leq \frac{\epsilon}{2n^2\sigma}$ as well. In the following rounds, this will allow us to accurately target any distribution in $\IRD(v^*)$ whenever $\norm{ v^* - v_t} \leq \frac{\epsilon}{2n^2L\sigma}$.

\begin{lemma}\label{lemma:ss-targeting}
Suppose we have estimates $F_i$ which satisfy the scale-bounded constraints for $x_t$ at $v_t$, with $\abs{ x_{t,i} - v_{t,i} } \leq \lambda$, and further that $\abs{f_i(v_t) - F_i} \leq \frac{\epsilon}{\sigma n^2}$ for each $i$. Then, the generated distribution $y_t$ when targeting $x_t$  according to $\{F_{t,i}\}$  satisfies $\norm{y_t - x_t}_1 \leq \frac{\epsilon}{n}$.
\end{lemma}
\begin{proof}
The generated distribution is given by
\begin{align*}
    y_i =&\; x_i + \frac{f_i(v_t) - F_i}{F_i}x_i 
\end{align*}
and so 
\begin{align*}
    \abs{y_i - x_i} \leq&\; \frac{\epsilon}{\sigma n} \cdot \frac{x_i}{F_i}  \leq \frac{\epsilon}{n}.
\end{align*}
\end{proof}

\paragraph{Online gradient descent.}

We treat each batch of $S$ rounds as a single step for $\ogd$ (with $x_s$ as the point chosen  by $\ogd$), and show that each of the following invariants is maintained across every step with high probability:
\begin{enumerate}
    \item We complete each learning stage with estimates $F_i$ satisfying $\abs{f_i(v_t) - F_i} \leq \frac{\epsilon}{2\sigma n^2}$;
    \item $\abs{f_i(v_t) - F_i} \leq \frac{\epsilon}{\sigma n^2}$ holds in every round where $x_s$ is targeted;
    \item Each gradient estimate satisfies $\norm{ \tilde{\nabla}_t - r_s}_2 \leq \frac{4\epsilon}{n}$, $\norm{ \tilde{\nabla}_t }_2 \leq 2\sqrt{n}$, and $\E\brackets{\norm{\tilde{\nabla}_s }^2} \leq O(n^2)$.
    \item We begin each step with $\norm{v_t - x_s}_1 \leq \frac{4\epsilon}{n}$;
    \item We complete each step with $\norm{v_t - x_s}_1 \leq \frac{2\epsilon}{n}$;
    \item The expected choice distribution $x_t$ in each round of the step satisfies $\norm{x_t - x_s}_1 \leq \frac{4\epsilon}{n}$ and $x_t \in \Delta^{\phi}(n)$.
\end{enumerate}

\paragraph{(1.)} This holds whenever $\norm{ v^* - v_t} \leq \frac{\epsilon}{2n^2L\sigma}$, provided that each update results in $\abs{f_i(v^*) - F_i} \leq \frac{\epsilon}{2n^2\sigma}$ at the time of completion. The latter follows along the lines of the initial learning stage; while we may initially have accuracy of only $\frac{\epsilon}{n^2\sigma}$ accuracy, the learning occurs within $\tilde{O}(n^8 / \epsilon^2) \leq T^c/ L^2$ rounds to a target accuracy of $\frac{\epsilon}{4\sigma n^2}$ for $\tilde{v}$, and so the total drift throughout learning can be bounded by $\frac{\epsilon}{4n^2L\sigma}$ as the fraction of memory in which the drift applies is bounded by Lemma \ref{lemma:drift-bound-2} (assuming a sufficiently large constant lower bound for $L$, without loss of generality). This further implies the desired accuracy for $F_i$ at at the updated $v^*$.
\paragraph{(2.)} This holds given that we re-learn whenever $\norm{ v^* - v_t} \geq \frac{\epsilon}{2n^2L\sigma}$.
\paragraph{(3.)}
Each $x_s$ lies within $\Delta^{\phi}_{\epsilon}(n)$, and any total drift of the memory vector outside of $\Delta^{\phi}_{\epsilon}(n)$ can be bounded by $\frac{\epsilon}{2n} \cdot v_t$, and so our target distribution always lies within $\Delta^{\phi}(n)$, as well as $\IRD(v_t, M)$ for $\lambda\phi \gg \epsilon$ by Lemma \ref{lemma:pseudoinc-ird}. 
Given (6.), we will have that $\E\brackets{\norm{\tilde{\nabla}_s- r_s }} \leq \frac{4\epsilon}{\sqrt{n}}$, similarly to as in the analysis of $\dbg$, where $r_s = \sum_{S} r_t/S$. Further, with each $x_t \geq \frac{\phi}{n} \geq \frac{\epsilon}{n}$, we will also have $\E\brackets{\norm{\tilde{\nabla}_s }^2} \leq O(n^2)$ as we had for $\dbg$ as well, with a norm bound of $2\sqrt{n}$ holding with high probability.

\paragraph{(4.)} 
This holds as $\norm{\tilde{\nabla}_s}_2 \eta \leq 2\sqrt{n}\eta \leq \frac{\epsilon}{2n}$.

\paragraph{(5.)} When $\norm{v_t - x_s}_2 \leq \frac{4\epsilon}{n}$ but $\norm{v_t - x_s}_2 \geq \frac{2\epsilon}{n}$, our time spent targeting $x_s$ (up to $\frac{\epsilon}{n}$, by Lemma \ref{lemma:ss-targeting}) is sufficient to decrease the distance by at least $\frac{3\epsilon}{8n^2L\sigma}$ with high probability before drifting more than $\frac{\epsilon}{2n^2L\sigma}$ from $v^*$, and following the potential drift of $\frac{\epsilon}{4n^2L\sigma}$ during re-learning each $F_i$, we remain closer by $\frac{\epsilon}{8n^2L\sigma}$ with high probability. Thus, at most $16n\sigma$ stages are needed to reach within $\frac{2\epsilon}{n}$ from $x_s$, which holds $S = \tilde{O}(n^{3/2} T^c)$ and $\tilde{O}(T^c/L^2)$ total between each update to $v^*$. 

\paragraph{(6.)} This follows from the accuracy guarantees of each learning stage as well as the drift bounds applied to each $\tilde{v}$.

\paragraph{Regret bound.}
Given each of these, we can analyze our regret in accordance with the bounds for Online Gradient Descent, as well as the  error resulting from the above approximations to an exact execution of $\ogd$. This gives us a total regret bound for our algorithm of: 

\begin{align*}
    \Reg_{\Delta^{\phi}(n)} =&\; \underbrace{S \parens{ \eta \cdot \sum_{s=1}^{T/S} \norm{ \tilde{\nabla}_{s}}^2 + \frac{\sqrt{2}}{\eta}  }}_{\text{$\ogd$ regret over $\Delta^{\phi}_{\epsilon}(n)$, $T/S$ steps}} + \underbrace{ \sum_{s=1}^{T/S} \sum_{h=t}^{t+S-1} \sqrt{n}\parens{ \norm{ r_s - \tilde{\nabla}_{s}} + \norm{x_h - x_s } } }_{\text{gradient and reward error}}   \\
    &\;+ \underbrace{\parens{ \max_{ \Delta^{\phi}(n) \times \Delta^{\phi}_{\epsilon}(n)} \sum_t r_t^{\top} x^* -  \sum_t r_t^{\top} x  }}_{\text{target set imprecision vs. $\Delta^{\phi}_n$}} + \underbrace{\tilde{O}\parens{T^c}}_{\text{burn-in and initial learning}}\\
    =&\; \tilde{O}\parens{ n^{2} T^{1/2 + c/2} + n^{7/2}L (T^{1/2 + c/2} + T^{1 - c/2} ) +  n^4 L T^{1 - c/2} + T^c }.
\end{align*}

\section{Proof of Theorem \ref{thm:np-hardness}}

\begin{proof}
We reduce to an instance of Maximum Independent Set (MIS). The MIS problem is Poly-APX-Hard (see e.g.\ \cite{10.5555/1540612}), and so there is no constant factor polynomial time approximation algorithm unless $\textsc{P}=\textsc{NP}$. Given a graph $G$ with $n-1$ vertices (which we can assume to have a unique maximum independent set $V^*$ with at most $(n-1)/2$ vertices), we construct a circuit for a preference model $M$ for $n$ items where identifying the optimal item distribution contained in its $\IRD$ set corresponds to identifying the maximum independent set $S^*$. For our reward function, we assume that item 1 yields a constant reward of 1, with all other items yielding a reward of 0. The objective is then to maximize the probability of item 1 being selected, but the score $f_1(v) \geq \lambda$ will only be maximized when the memory mass not placed on item $i$ is (near-)uniformly allocated across items corresponding to the maximum independent set (the other scoring functions can be constant at $ \lambda$). 

We describe our preference model in terms of a circuit for $f_1$ with both arithmetic and Boolean gates, where all other functions $f_i$ are presumed to be constant at $\lambda$, and note that translation to a pure Boolean circuit is feasible with at most polynomial blowup.
Let $N = \{2,\ldots,n\}$ correspond to the vertices of $G$.
For any memory vector $v$, let $V = \{j > 1 : v_j \geq \frac{1}{n} \}$. 
Our function $f_1$ is given by
\begin{align*}
    f_1(v) =&\; \begin{cases}
        \frac{\abs{V}(1 - \lambda)}{n-1} + \lambda & V \text{ is an independent set in } G \\
        \lambda & V \text{ is not independent in } G
    \end{cases}.
\end{align*}
To construct a circuit for $f_1$, first we include gates $g_j$ which output 1 for each $j$ if and only if $v_j \geq \frac{1}{n}$ (and 0 otherwise). Given these gates, we also include gates $e_{ij}$ for each edge in the graph which output 1 if and only if $i$ has an edge with $j$ (where $1 - e_{ij}$ then denotes no edge). We can construct a gate $g_{\text{valid}}$ which outputs 1 if and only if our set $V$ is an independent set by taking the AND of all gates $( e_{ij} \oplus ( g_i \land g_j))$, as well as a gate $g_{\text{count}}$ which gives the count of proposed independent nodes by summing over gates $g_i$. Taking the product of $g_{\text{valid}}$ and $g_{\text{count}}$ then yields a counter for the size of $V$ if $V$ is independent, and 0 otherwise, which can then be arithmetically scaled to yield $f_1(v)$.

Let $\lambda = 1/((n-1)^2/k + 1)$. Given an independent set $V$, the highest reward obtainable by a memory vector $v \in \IRD(v, M)$ which corresponds to $V$ under $f_1$ can be expressed via menu times, where we must have that
\begin{align*}
    \frac{ k \frac{ v_1} { f_1(v) } }{ \frac{ v_1} { f_1(v) } + \frac{(1 - v_1)}{\lambda} }  \leq&\; 1
\end{align*}
by Lemma \ref{lemma:ird-menus}. Solving $(k-1)v_1 = (1 - v_1)\parens{ 1  + \frac{\abs{V} (1 - \lambda)}{ \lambda (n-1))} }$ for $v_1$ gives us that
\begin{align*}
    v_1 =&\; \frac{\frac{1}{k-1}  + \frac{\abs{V} (1 - \lambda)}{ \lambda (n-1) (k-1)}}{1 + \frac{1}{k-1}  + \frac{\abs{V} (1 - \lambda)}{ \lambda (n-1) (k-1)}} \\ 
    =&\; \frac{\frac{k \abs{V}}{n-1} + 1}{ \frac{k \abs{V}}{n-1}  + k}  \tag{$\lambda(n-1)^2/k = 1 - \lambda$}. \\ 
\end{align*}
This is at most $1/2$ for any $\abs{V} < (n-1)/2$, and so further it remains possible to allocate menu time which yields $v_i \geq \frac{1}{n}$ for each $i \in V$, yielding that $v_1$ the maximum feasible reward for some $v \in \IRD(v, M)$ when $\abs{V}$ is the size of the largest independent set. 
The gradient of $v_1$ in terms of $\abs{V}$ is given by
\begin{align*}
   \frac{\partial}{\partial \abs{V}} \frac{\abs{V} + (n-1)/k}{  \abs{V} + (n-1)}  =&\; \frac{  (n-1)(1 - 1/k) }{( \abs{V} + (n-1))^2}
\end{align*}
which is $\Theta(1/n)$ for $\abs{V} \in [1, n/2]$. Any polynomial time algorithm which approximates $v_1$ to within a $O(1/n)$ factor must result in a memory vector $v'$ which corresponds to an independent set $V'$ which approximates $V^*$ within a constant factor, which would yield a polynomial time constant factor approximation algorithm for Maximum Independent Set.
\end{proof}

\section{Omitted Proofs for Section \ref{sec:smallgamma}}

\subsection{Proof of Theorem \ref{thm:short-ss-alg}}

\begin{proof} The key element of our analysis is to analyze the convergence of item frequencies during windows of length $t_{\text{hold}}$ when a fixed target item $i$ is held constant in the menu. For a given such window of length $t_{\text{hold}} = O(\frac{1}{\alpha^4(1-\gamma)})$, we can ensure that the accumulated reward approaches its expectation under the current reward distribution to within $\alpha$.
As we choose the $k-1$ smallest weights in memory, the total weight of items in memory other than $i$ is at most $\frac{(k-1)(1 - v_{i})}{n}$; given a current memory vector $v$, the probability of selecting item $i$ from a menu $K_t$ is given by:
\begin{align*}
    \Pr[i^* \text{ selected}] =&\; \frac{f_i(v)}{\sum_{K_t} f_j(v_i)}\\
    \geq&\; \frac{(1 - \lambda)v_i + \lambda}{(1 - \lambda)v_i + \lambda + \parens{\frac{(1 - v_i)(1 - \lambda)}{n} + \lambda}(k-1)\sigma^2 }
\end{align*}
by the pseudo-increasing property. Our approach will be to analyze the expectation of $v_{t,i}$ over time, with $E_t = \E[v_{t, i}]$, and show that it approaches $1 - \phi(n-1)/n$, equal to the probability at the corresponding vertex of the smoothed simplex. A challenge is that, given a current expectation, there are many possible allocations of probabilities to values of $v_{t,i}$ which yield $E_t$. A second derivative test shows that the above probability function is concave for positive $v_i$ when $\sigma^2 \leq n/(2k)$; note that both the numerator and denominator are positive and increasing linearly in $v_i$, and that the numerator is always smaller but grows faster in $v_i$. As such, we can apply Jensen's inequality and restrict our consideration to the extremal case where the expectation $E_t$ is entirely composed of trials in which $v_{t,i}=0$ or $v_{t,i}=1$, which indeed occurs at $\gamma=0$. 
We can define $P_0$ and $P_1$ as lower bounds on selection probabilities for each case:
\begin{align*}
    \Pr[i^* \text{ selected} | v_{t,i} = 1] \geq&\; \frac{1 }{1 + (k-1)\lambda\sigma^2 } \\
    \geq&\; 1 - k\lambda\sigma^2  \\
    :=&\; P_1; \\
\end{align*}
\begin{align*}
    \Pr[i^* \text{ selected} | v_{t,i} = 0] \geq&\; \frac{\lambda}{ \lambda + (\frac{1-\lambda}{n} + \lambda)(k-1)\sigma^2 } \\
    =&\; \frac{1 }{1 + \parens{\frac{1- \lambda}{n\lambda}+1} (k-1)\sigma^2 } \\
    \geq&\; \frac{1 }{1 + \parens{k\sigma^2 +1} (k-1)\sigma^2 } \\
    \geq&\; \frac{1}{2\sigma^2 k^4}  \\
    := P_0.
\end{align*}
As such, we have that
\begin{align*}
   E_{t+1} =&\;  \E[v_{t+1} | E_t] \\
   \geq&\; (1 - \gamma)\parens{ E_t \cdot \Pr[i^* \text{ selected} | v_{t,i} = 1] + (1 - E_t) \cdot \Pr[i^* \text{ selected} | v_{t,i} = 0]   } + \gamma E_t \\
   \geq&\; (1 - \gamma)\parens{ E_t \cdot P_1 + (1 - E_t) \cdot P_0   } + \gamma E_t. 
\end{align*}
We can solve for $E_t^*$ such that $E_{t+1} = E_{t}$, i.e.\ where $E_t \cdot P_1 + (1 - E_t) \cdot P_0 = E_t$, as:
\begin{align*}
    E_t^* =&\; \frac{1}{1+ 2\lambda \sigma^6k^3} \\
    \geq&\; 1 - 2\sigma^6 k^3\lambda \\
    \geq&\; 1 - \phi(n-1)/n \tag{$n\geq 3$}
\end{align*}
and further for a value $E_t^{\alpha}$ such that $E_{t+1} \geq (1 - \gamma)(E_t + \alpha) + \gamma E_t$ as:
\begin{align*}
    E_t^{\alpha} =&\; E_t^* - 2\sigma^4k^2 \alpha.
\end{align*}
Note that the rate of growth of $E_{t+1}$ is decreasing in $t$, and eventually reaches a fixed point; given that the rate of growth of $E_t$ is linear in $\alpha$ when within $O(\alpha)$ of $E_t^*$, the cumulative number of rounds required to reach $E_t^* - O(\alpha)$ is at most $O(\frac{1}{\alpha(1- \gamma)})$. If we continue after for $O(\frac{1}{\alpha^2(1- \gamma)})$ rounds, these first rounds contribute at most $\alpha$ to the total summed expectation for the fraction of rounds in which item $i$ is selected is at least $E_t^* - \alpha$; the fraction of each other item played also quicly approaches $\frac{1 - E^*_t}{n-1}$ in expectation. 

Treat each such batch of $O(\frac{1}{\alpha^2(1- \gamma)})$ rounds as a trial, and repeat for $\tilde{O}(1/\alpha^2)$ trials, resulting in a total of $t_{\text{hold}} = O(\frac{1}{\alpha^4(1 - \gamma)})$ steps. We can treat each trial as independent, as resetting the memory vector to some lower value of $v_i$ can only decrease expected reward under the pessimistic lower bounds we consider. By a Hoeffding bound, we have that the reward is within $\alpha$ from the expectation under the current distribution and the ``arm'' of the $\phi$-smoothed simplex corresponding to $i$.
To complete the analysis, observe that our total regret (using the $\tilde{O}(T^{1/2})$ bound for EXP3) is given by:
\begin{align*}
    \Reg_{\Delta^{\phi}(n)}(\A_3 ; T) =&\; \tilde{O}(t_{\text{hold}}\cdot \parens{\frac{T}{t_{\text{hold}}}}^{1/2} + \alpha T ) \\
    =&\; \tilde{O}(\frac{T^{1/2}}{\alpha^2} + \alpha T ) \\
    =&\; \tilde{O}(T^{5/6})
\end{align*}
upon setting $\alpha = O(T^{-1/6})$.

\end{proof}

\subsection{Proof of Theorem \ref{thm:short-eird-hard}}
\begin{proof}
We consider a set of models and reward functions where item $1$ yields a reward of 1 in each round, with all other items yielding a reward of 0. For $\gamma = \frac{1}{2^c}$ for some constant $c > 1$, note that the weight of any step in memory is larger than the sum of weights of all preceding steps, and thus a memory vector $v_t$ exactly encodes the history of item selections for the first $t-1$ rounds. Let $h = \log_{2^c}(n)$; For $t$ sufficiently larger than $h$, the sum of weights of steps $1$ through $t-h$ is $\Theta(1/n)$. We will consider states $s$ which are subsets of the space of memory vectors corresponding to each possible history truncated to the previous $h$ steps, and which are bounded apart by a distance of at least $O(1/n)$. We will abuse notation and represent each memory vector $v_t$ as its rounded state $s_t$. The behavior of the memory model is constant over each state, and smoothly interpolates between states; the model can be defined arbitrarily for infeasible memory vectors to satisfy $L=\poly(n)$ Lipschitzness. Our process for generating $\mathcal{M}$ is as follows:
\begin{itemize}
    \item Let $k = n/2$;
    \item Let $\lambda = \frac{1}{n - k + 1}$; 
    \item For each $s \in [n]^h$, let $G_s$ be a set of $k-2$ items sampled uniformly at random from $\{2,\ldots,n\}$;
    \item Let $f_i(s) = \lambda$ if $i = 1$ or $i \in G_s$, and $f_i(s) = 1$ otherwise.
\end{itemize}
Observe that the optimal strategy $\pi^*$ at $s$ is to include item 1, each item in $G_s$, and any arbitrary final item. Note that each of the $k-1$ items with score $\lambda$ is selected with probability
\begin{align*}
    \Pr[i \text{ chosen} | f_i(s) = \lambda, \pi^* \text{ played}] =&\; \frac{\lambda}{1 + (k-1)\lambda} \\
    =&\; \frac{1}{n},
\end{align*}
and so the expected reward per round is 1 as well. Note that $\pi^*$ is consistent with a menu distribution which chooses the final item (after 1 and $G_s$) uniformly at random, which generates the uniform distribution. As such, the uniform distribution lies inside $\EIRD$ (t is straightforward to define scores for infeasible memory vectors such that feasibility holds for any $v \in \Delta(n)$). We can also see that any menu inconsistent with $\pi^*$ has expected reward at most:
\begin{align*}
    \frac{\lambda}{2 + (k-2)\lambda} =&\; \frac{\frac{1}{n - k + 1}}{\frac{2n - 2k + 2}{n - k + 1}  + \frac{k - 2}{n - k + 1}} \\
    =&\; \frac{3}{4n},
\end{align*}
as some item not in $G_s$ must be included. 
To lower bound the regret of any algorithm, consider an arbitrary time $t$ and history of item selections. By time $t$, at most distinct states have been observed thus far. Consider the following cases:
\begin{itemize}
    \item The algorithm plays a menu consistent with $\pi^*$ in every step from $t$ to $t+h-1$;
    \item The algorithm plays a menu inconsistent with $\pi^*$ at some step from $t$ to $t+h-1$. 
\end{itemize}
Suppose $t$ is less than $T = \frac{1}{2} \cdot (\frac{n}{2})^h = O(n^{\log(n)})$. In the former case, there is a uniform distribution over $\frac{n}{2}$ items chosen by the agent at each round, and so the maximum probability of any given state is at most $(\frac{n}{2})^h$.
Given that the set $G_s$ is generated independently at random for each state, an algorithm has no information about $G_s$ for unvisited states, and thus cannot improve expected reward beyond that obtained by choosing a random hypothesis for $G_s$, which incurs $O(\frac{1}{n})$ regret at round $t+h$. In the latter case, the step in which a menu inconsistent with $\pi^*$ is played additionally incurs $O(\frac{1}{n})$ regret.
Each event occurs once at least once in expectation every $h$ rounds while $t < T$, and thus any algorithm must have $O(\frac{1}{nh})$ expected regret per round up to $T$.
\end{proof}

\end{document}